\documentclass[journal, 10pt]{IEEEtran}

\usepackage{cite}
\usepackage{amsmath,amssymb,amsfonts}
\usepackage{algorithmic}
\usepackage{graphicx}
\usepackage{textcomp}
\usepackage{amsthm}
\usepackage{mathtools}
\usepackage{algorithm}
\usepackage{algorithmic}
\usepackage{color}
\usepackage{booktabs}

\allowdisplaybreaks

\DeclareMathOperator*{\argmax}{arg\,max}

\newtheorem{remark}{Remark}
\newtheorem{assumption}{Assumption}
\newtheorem{lemma}{Lemma}
\newtheorem{theorem}{Theorem}
\newtheorem{corollary}{Corollary}
\newtheorem{defn}{Definition}

\usepackage{mathtools}

\newcommand{\expect}{\operatorname{E}\expectarg}
\DeclarePairedDelimiterX{\expectarg}[1]{[}{]}{%
  \ifnum\currentgrouptype=16 \else\begingroup\fi
  \activatebar#1
  \ifnum\currentgrouptype=16 \else\endgroup\fi
}

\newcommand{\innermid}{\nonscript\;\delimsize\vert\nonscript\;}
\newcommand{\activatebar}{%
  \begingroup\lccode`\~=`\|
  \lowercase{\endgroup\let~}\innermid 
  \mathcode`|=\string"8000
}

\newcommand{\bs}{\boldsymbol}
\newcommand{\mr}{\mathrm}

\newcommand{\comment}[1]{}

\ifodd 0
\newcommand{\rr}[1]{{\color{blue}#1}}
\else
\newcommand{\rr}[1]{#1}
\fi

\ifodd 1
\newcommand{\com}[1]{{\color{red}Comment: #1}}
\else
\newcommand{\com}[1]{}
\fi

\ifodd 1
\newcommand{\add}[1]{{\color{blue}#1}}
\else
\newcommand{\add}[1]{}
\fi

\ifodd 0
\newcommand{\rev}[1]{{\color{blue}#1}}
\else
\newcommand{\rev}[1]{#1}
\fi

\ifodd 1
\newcommand{\er}[1]{{\color{green}#1}}
\else
\newcommand{\er}[1]{}
\fi

\ifodd 1

\else

\fi

\title{Multi-objective Contextual Multi-armed Bandit with a Dominant Objective}
	\author{\IEEEauthorblockN{Cem Tekin,~\IEEEmembership{Member,~IEEE}, Eralp Tur\u{g}ay \\}
		\thanks{This work is copyrighted by the IEEE. It has been accepted for publication in IEEE Transactions on Signal Processing. See IEEE's electronic database under DOI: 10.1109/TSP.2018.2841822.}
		\thanks{C. Tekin and E. Tur\u{g}ay are with the Department of Electrical and Electronics Engineering, Bilkent University, Ankara, Turkey, 06800. Email: cemtekin@ee.bilkent.edu.tr, turgay@ee.bilkent.edu.tr.}
		\thanks{This work was supported by TUBITAK 2232 Grant 116C043 and was supported in part by TUBITAK 3501 Grant 116E229.}
		\thanks{A preliminary version of this work was presented in IEEE MLSP 2017 \cite{tekin2017MOCMAB}.}
	}

\begin{document}

\maketitle

\begin{abstract}
In this paper, we propose a new multi-objective contextual multi-armed bandit (MAB) problem with two objectives, where one of the objectives dominates the other objective. Unlike single-objective MAB problems in which the learner obtains a random scalar reward for each arm it selects, in the proposed problem, the learner obtains a random reward vector, where each component of the reward vector corresponds to one of the objectives and the distribution of the reward depends on the context that is provided to the learner at the beginning of each round. We call this problem contextual multi-armed bandit with a dominant objective (CMAB-DO). In CMAB-DO, the goal of the learner is to maximize its total reward in the non-dominant objective while ensuring that it maximizes its total reward in the dominant objective. In this case, the optimal arm given a context is the one that maximizes the expected reward in the non-dominant objective among all arms that maximize the expected reward in the dominant objective. First, we show that the optimal arm lies in the Pareto front. Then, we propose the multi-objective contextual multi-armed bandit algorithm (MOC-MAB), and define two performance measures: the 2-dimensional (2D) regret and the Pareto regret. We show that both the 2D regret and the Pareto regret of MOC-MAB are sublinear in the number of rounds. We also compare the performance of the proposed algorithm with other state-of-the-art methods in synthetic and real-world datasets. The proposed model and the algorithm have a wide range of real-world applications that involve multiple and possibly conflicting objectives ranging from wireless communication to medical diagnosis and recommender systems.
\end{abstract}

\begin{IEEEkeywords}
Online learning, contextual MAB, multi-objective MAB, dominant objective, multi-dimensional regret, Pareto regret.
\end{IEEEkeywords}

\section{Introduction}
\label{sec:intro}
With the rapid increase in the generation speed of the streaming data, online learning methods are becoming increasingly valuable for sequential decision making problems. Many of these problems,  including recommender systems \cite{li2010contextual,xu2016personalized}, medical screening \cite{song2016using}, cognitive radio networks \cite{gai2010learning,gai2014distributed} and wireless network monitoring \cite{le2014sequential} may involve multiple and possibly conflicting objectives. In this work, we propose a multi-objective contextual MAB problem with dominant and non-dominant objectives. 
For this problem, we construct a multi-objective contextual MAB algorithm named MOC-MAB, which maximizes the long-term reward of the non-dominant objective conditioned on the fact that it maximizes the long-term reward of the dominant objective. 

In this problem, the learner observes a multi-dimensional context in the beginning of each round. Then, it selects one of the available arms and receives a random reward vector, which is drawn from a fixed distribution that depends on the context and the selected arm.  
No statistical assumptions are made on the way the contexts arrive, and the learner does not have any a priori information on the reward distributions. 
The optimal arm for a given context is defined as the one that maximizes the expected reward of the non-dominant objective among all arms that maximize the expected reward of the dominant objective. 

The learner's performance is measured in terms of its regret, which measures the loss that the learner accumulates due to not knowing the reward distributions beforehand. We introduce two new notions of regret: the 2D regret and the Pareto regret. The 2D regret is a vector whose $i$th component corresponds to the difference between the expected total reward of an oracle in objective $i$ that selects the optimal arm for each context and that of the learner by time $T$. On the other hand, the Pareto regret measures sum of the distances of the arms selected by the learner to the Pareto front.  For this, we extend the Pareto regret proposed in \cite{drugan2013designing} to take into account the dependence of the Pareto front on the context.

We prove that MOC-MAB achieves $\tilde{O} ( T^{(2\alpha+d)/(3\alpha+d)} )$ 2D regret, where $d$ is the dimension of the context and $\alpha$ is a constant that depends on the similarity information that relates the distances between contexts to the distances between expected rewards of an arm. This shows that MOC-MAB is average-reward optimal in the limit $T \rightarrow \infty$ in both objectives. 
We also show that the optimal arm lies in the Pareto front, and MOC-MAB also achieves $\tilde{O} ( T^{(2\alpha+d)/(3\alpha+d)} )$ Pareto regret. 
Then, we argue that it is possible to make the Pareto regret of MOC-MAB $\tilde{O} ( T^{(\alpha+d)/(2\alpha+d)} )$ by adjusting its parameters, such that the Pareto regret becomes order optimal up to a logarithmic factor \cite{lu2010contextual}, but this comes at an expense of making the regret in the non-dominant objective of MOC-MAB linear in the number of rounds.

To the best of our knowledge, our work is the first to formulate a contextual multi-objective MAB problem and prove sublinear bounds on the 2D regret and the Pareto regret. \rr{Different from the conference version \cite{tekin2017MOCMAB}, in this paper we (i) consider the Pareto regret in addition to the 2D regret, (ii) connect our notion of optimality with lexicographic optimality, (iii) provide a high probability bound on the 2D regret, (iv) show how MOC-MAB can be extended to deal with periodically changing expected arm rewards, (v) discuss how CMAB-DO can be extended for more than two objectives, (vi) provide numerical results on multichannel communication and display advertising applications. Our results show that MOC-MAB outperforms its competitors, which are not specifically designed to deal with problems involving dominant and non-dominant objectives. Moreover, the journal version includes all the proofs.}

The rest of the paper is organized as follows. Related work is given in Section \ref{sec:rwork}. Problem formulation, definitions of the 2D regret and the Pareto regret, and possible applications of CMAB-DO are given in Section \ref{sec:pdesc}. MOC-MAB is introduced in Section \ref{sec:alg}, and its regrets are analyzed in Section \ref{sec:regA}. \rr{How MOC-MAB can be extended to work under dynamically changing reward distributions and how CMAB-DO can be extended to capture more than two objectives are discussed in Section \ref{sec:extensions}}. Illustrative results are presented in Section \ref{sec:sims}, and concluding remarks are provided in Section \ref{sec:cnc}.

\section{Related work} \label{sec:rwork}

\begin{table*}[t]
	\caption{Comparison of the regret bounds and assumptions in our work with the related works. }
	\label{table:related}
	\centering
	\begin{tabular}{l l p{1.2cm}  p{1.2cm}  p{1.7cm}    p{2cm}  p{1.8cm}}
		\toprule
		MAB algorithm & Regret bound &  Multi-objective & Contextual  & Linear rewards &  Similarity assumption  \\
		\midrule
		Contextual Zooming \cite{slivkins2014contextual}   & $ \tilde{O}(T^{1-1/(2+d_z)})$  & No & Yes      & No  & Yes   \\
		Query-Ad-Clustering \cite{lu2010contextual}             & $ \tilde{O}(T^{1-1/(2+d_c)})$ & No & Yes    & No  & Yes\\
		SupLinUCB \cite{chu2011contextual}       & $\tilde{O}(\sqrt{T})$  & No & Yes    & Yes & No\\
		Pareto-UCB1 \cite{drugan2013designing}     & $O(\log(T))$ & Yes & No   & No & No \\
		Scalarized-UCB1\cite{drugan2013designing}      & $O(\log(T))$ & Yes & No    & No & No\\
		MOC-MAB (our work)  & $ \tilde{O}(T^{(2\alpha +d)/(3\alpha +d)})$ (2D and Pareto regrets) & Yes & Yes     & No  & Yes   \\
		 & $ \tilde{O}(T^{(\alpha +d)/(2\alpha +d)})$ (Pareto regret only)  &  &      &   &   \\
		\bottomrule
	\end{tabular}
\end{table*}

In the past decade, many variants of the classical MAB have been introduced (see \cite{bubeck2012regret} for a comprehensive discussion). 
Two notable examples are contextual MAB \cite{langford2007epoch, slivkins2014contextual,agarwal2014taming} and multi-objective MAB \cite{drugan2013designing}. While these examples have been studied separately in prior works, in this paper we aim to fuse contextual MAB and multi-objective MAB together. Below, we discuss the related work on the classical MAB, contextual MAB and multi-objective MAB. The differences between our work and related works are summarized in Table \ref{table:related}.

\subsection{The Classical MAB}

The classical MAB involves $K$ arms with unknown reward distributions. The learner sequentially selects arms and observes noisy reward samples from the selected arms. The goal of the learner is to use the knowledge it obtains through these observations to maximize its long-term reward. 
For this, the learner needs to identify arms with high rewards without wasting too much time on arms with low rewards. In conclusion, it needs to strike the balance between exploration and exploitation.

A \rr{thorough} technical analysis of the classical MAB is given in \cite{lai1}, where it is shown that $O(\log T)$ regret is achieved asymptotically by index policies that use upper confidence bounds (UCBs) for the rewards. This result is tight in the sense that there is a matching asymptotic lower bound. Later on, it is shown in \cite{agrawal1} that it is possible to achieve $O(\log T)$ regret by using index policies constructed using the sample means of the arm rewards. The first finite-time logarithmic regret bound is given in \cite{auer2002finite}. Strikingly, the algorithm that achieves this bound computes the arm indices using only the information about the current round, the sample mean arm rewards and the number of times each arm is selected. This line of research has been followed by many others, and new algorithms with tighter regret bounds have been proposed \cite{garivier2011kl}.

\subsection{The Contextual MAB}

In the contextual MAB, different from the classical MAB, the learner observes a context (side information) at the beginning of each round, which gives a hint about the expected arm rewards in that round. The context naturally arises in many practical applications such as social recommender systems \cite{tekin2014distributed}, medical diagnosis \cite{tekin2016confidence} and big data stream mining \cite{cem2013deccontext}.  
Existing work on contextual MAB can be categorized into three based on how the contexts arrive and how they are related to the arm rewards. 

The first category assumes the existence of similarity information (usually provided in terms of a metric) that relates the variation in the expected reward of an arm as a function of the context to the distance between the contexts. For this category, no statistical assumptions are made on how the contexts arrive. However, given a particular context, the arm rewards come from a fixed distribution parameterized by the context.

 This problem is considered in \cite{lu2010contextual}, and the Query-Ad-Clustering algorithm that achieves $O(T^{1-1/(2+d_c) +\epsilon})$ regret for any $\epsilon >0$ is proposed, where $d_c$ is the covering dimension of the similarity space. In addition, $\Omega (T^{1-1/(2+d_p)-\epsilon} )$ lower bound on the regret, where $d_p$ is the packing dimension of the similarity space, is also proposed in this work.
 The main idea behind Query-Ad-Clustering is to partition the context set into disjoint sets and to estimate the expected arm rewards for each set in the partition separately.
 A parallel work \cite{slivkins2014contextual} proposes the contextual zooming algorithm which partitions the similarity space non-uniformly, according to both sampling frequency and rewards obtained from different regions of the similarity space. It is shown that contextual zooming achieves $\tilde{O}(T^{1-1/(2+d_z)})$ regret, where $d_z$ is the zooming dimension of the similarity space, which is an optimistic version of the covering dimension that depends on the size of the set of near-optimal arms.

 In this contextual MAB category, reward estimates are accurate as long as the contexts that lie in the same set of the context set partition are similar to each other. However, when dimension of the context is high, the regret bound becomes almost linear. This issue is addressed in \cite{tekin2015releaf}, where it is assumed that the arm rewards depend on an unknown subset of the contexts, and it is shown that the regret in this case only depends on the number of relevant context dimensions.  

The second category assumes that the expected reward of an arm is a linear combination of the elements of the context. For this model, LinUCB algorithm is proposed in \cite{li2010contextual}. A modified version of this algorithm, named SupLinUCB, is studied in \cite{chu2011contextual}, and is shown to achieve $\tilde{O}(\sqrt{Td})$ regret, where $d$ is the dimension of the context. Another work \cite{valko2013finite} considers LinUCB and SupLinUCB with kernel functions and proposes an algorithm with $\tilde{O} (\sqrt{T \tilde{d}} )$ regret, where $\tilde{d}$ is the effective dimension of the kernel feature space.

The third category assumes that the contexts and arm rewards are jointly drawn from a fixed but unknown distribution. For this case, the Epoch-Greedy algorithm with $O(T^{2/3})$ regret is proposed in \cite{langford2007epoch}, and more efficient learning algorithms with $\tilde{O}(T^{1/2})$ regret are developed in \cite{agarwal2014taming} and \cite{dudik2011efficient}.

Our problem is similar to the problems in the first category in terms of the context arrivals and existence of the similarity information. 

\subsection{The Multi-objective MAB}

In the multi-objective MAB, the learner receives a multi-dimensional reward in each round. Since the rewards are no longer scalar, the definition of a benchmark to compare the learner against becomes obscure. 
Existing work on multi-objective MAB can be categorized into two: the Pareto approach and the scalarized approach.

In the Pareto approach, the main idea is to estimate the Pareto front set which consists of the arms that are not dominated by any other arm. Dominance relationship is defined such that if the expected reward of an arm $a^*$ is greater than the expected reward of another arm $a$ in at least one objective, and the expected reward of the arm $a$ is not greater than the expected reward of the arm $a^*$ in any objective, then the arm $a^*$ dominates the arm $a$.
This approach is proposed in \cite{drugan2013designing}, and a learning algorithm called Pareto-UCB1 that achieves $O(\log T)$ Pareto regret is proposed. Essentially, this algorithm computes UCB indices for each objective-arm pair, and then, uses these indices to estimate the Pareto front arm set, after which it selects an arm randomly from the Pareto front set.
A modified version of this algorithm where the indices depend on both the estimated mean and the estimated standard deviation is proposed in \cite{yahyaa2014knowledge}.  
Numerous other variants are also considered in prior works, including the Pareto Thompson sampling algorithm in \cite{yahyaa2015thompson} and the Annealing Pareto algorithm in \cite{yahyaa2014annealing}.

On the other hand, in the scalarized approach \cite{drugan2013designing,drugan2014scalarization}, a random weight is assigned to each objective at each round, from which for each arm a weighted sum of the indices of the objectives are calculated. In short, this method turns the multi-objective MAB into a single-objective MAB. For instance, Scalarized UCB1 in \cite{drugan2013designing} achieves 
$O(S'\log(T/S'))$ scalarized regret where $S'$ is the number of scalarization functions used by the algorithm.

The regret notion used in the Pareto and the scalarized approaches are very different from our 2D regret notion. In the Pareto approach, the regret at round $t$ is defined as the minimum distance that should be added to \rr{the} expected reward vector of the chosen arm at round $t$ to move the chosen arm to the Pareto front. On the other hand, scalarized regret is the difference between scalarized expected rewards of the optimal arm and the chosen arm. Different from these definitions, which define the regret as a scalar quantity, we define the 2D regret as a two-dimensional vector. Hence, our goal is to minimize a multi-dimensional regret measure conditioned on the fact that we minimize the regret in the dominant objective. We show that by achieving this, we also minimize the Pareto regret.

In addition to the works mentioned above, several other works consider multi-criteria reinforcement learning problems, where the rewards are vector-valued \cite{gabor1998multi,mannor2004geometric}.

\section{Problem Description} \label{sec:pdesc}

\subsection{System Model}

The system operates in a sequence of rounds indexed by $t \in \{1,2,\ldots\}$. At the beginning of round $t$, the learner observes a $d$-dimensional context denoted by $x_t$. Without loss of generality, we assume that $x_t$ lies in the context set ${\cal X} := [0,1]^d$. After observing $x_t$ the learner selects an arm $a_t$ from a finite set ${\cal A}$, and then, observes a two dimensional random reward $\bs{r}_t = (r^1_t, r^2_t)$ that depends both on $x_t$ and $a_t$.
Here, $r^1_t$ and $r^2_t$ denote the rewards in the dominant and the non-dominant objectives, respectively, and are given by
$r^1_t = \mu^{1}_{a_t}(x_t) + \kappa^1_t$ and $r^2_t = \mu^{2}_{a_t}(x_t) + \kappa^2_t$, where $\mu^{i}_{a}(x)$, $i \in \{1,2\}$ denotes the expected reward of arm $a$ in objective $i$ given context $x$, and the noise process $\{ (\kappa^1_t, \kappa^2_t) \}$ is such that the marginal distribution of $\kappa^i_t$, $i \in \{1,2\}$ is conditionally 1-sub-Gaussian,\footnote{Examples of 1-sub-Gaussian distributions include the Gaussian distribution with zero mean and unit variance, and any distribution defined over an interval of length $2$ with zero mean \cite{abbasi2011improved}. Moreover, our results generalize to the case when $\kappa^i_t$ is conditionally $R$-sub-Gaussian for $R\geq1$. This only changes the constant terms that appear in our regret bounds.} i.e., 
\begin{align*}
\forall \lambda \in \mathbb{R} ~~ 
\text{E} [e^{\lambda \kappa^i_t} | \bs{a}_{1:t}, \bs{x}_{1:t}, \bs{\kappa}^1_{1:t-1}, \bs{\kappa}^2_{1:t-1} ]
\leq \exp (\lambda^2/2) 
\end{align*}
where $\bs{b}_{1:t} := (b_1, \ldots, b_t)$.
The expected reward vector for context-arm pair $(x,a)$ is denoted by $\bs{\mu}_a(x) := (\mu^{1}_{a}(x), \mu^{2}_{a}(x))$.

%

%

The set of arms that maximize the expected reward for the dominant objective for context $x$ is given as ${\cal A}^*(x) := \argmax_{a \in {\cal A}} \mu^{1}_{a}(x)$. Let $\mu^{1}_{*}(x) := \max_{a \in {\cal A}} \mu^{1}_{a}(x)$ denote the expected reward of an arm in ${\cal A}^*(x)$ in the dominant objective.
The set of optimal arms is given as the set of arms in ${\cal A}^*(x)$ with the highest expected rewards for the non-dominant objective. Let $\mu^{2}_{*}(x) :=  \max_{a \in {\cal A}^*(x) }  \mu^{2}_{a}(x)$ denote the expected reward of an optimal arm in the non-dominant objective. We use $a^*(x)$ to refer to an optimal arm for context $x$.
\rr{The notion of optimality that is defined above coincides with lexicographic optimality \cite{ehrgott2005multicriteria}, which is widely used in multicriteria optimization, and has been considered in numerous applications such as achieving fairness in multirate multicast networks \cite{sarkar2002fair} and bit allocation for MPEG video coding \cite{hoang1997lexicographic}.}

We assume that the expected rewards are H\"older continuous in the context, which is a common assumption in the contextual MAB literature \cite{lu2010contextual,cem2013deccontext,tekin2016confidence}. 
\begin{assumption}\label{as:1}
There exists $L > 0$, $0 < \alpha \leq 1$ such that for all $i \in \left\{ 1,2 \right\} , a \in {\cal A}$ and $ x, x' \in {\cal X}$, we have
\begin{align*}
| \mu^{i}_{a}(x) - \mu^{i}_{a}(x') | \leq L \left\| x - x' \right\|^\alpha .
\end{align*}
\end{assumption}

\rr{
Since H\"older continuity implies continuity, for any non-trivial contextual MAB in which the sets of optimal arms in the first objective are different for at least two contexts, there exists at least one context $x \in {\cal X}$ for which ${\cal A}^*(x)$ is not a singleton. Let ${\cal X}^*$ denote the set of contexts for which ${\cal A}^*(x)$ is not a singleton. Since we make no assumptions on how contexts arrive, it is possible that majority of contexts that arrive by round $T$ are in set ${\cal X}^*$. This implies that contextual MAB algorithms that only aim at maximizing the rewards in the first objective cannot learn the optimal arms for each context. 
}

Another common way to compare arms when the rewards are multi-dimensional is to use the notion of Pareto optimality, which is described below. 
\begin{defn}[Pareto Optimality]
	(i) An arm $a$ is \emph{weakly dominated} by arm $a'$ given context $x$, denoted by $\bs{\mu}_{a}(x) \preceq \bs{\mu}_{a'}(x)$ or $\bs{\mu}_{a'}(x) \succeq \bs{\mu}_{a}(x)$, if $\mu_{a}^{i}(x) \leq \mu_{a'}^{i}(x), \forall i \in \{1,2\}$. \\
	(ii) An arm $a$ is \emph{dominated} by arm $a'$ given context $x$, denoted by $\bs{\mu}_{a}(x) \prec \bs{\mu}_{a'}(x)$ or $\bs{\mu}_{a'}(x) \succ \bs{\mu}_{a}(x)$, if it is weakly dominated and $\exists i \in \{1,2\}$ such that $\mu_{a}^{i}(x) < \mu_{a'}^{i}(x)$. \\
	(iii) Two arms $a$ and $a'$ are incomparable given context $x$, denoted by $\bs{\mu}_{a}(x) || \bs{\mu}_{a'}(x)$, if neither arm dominates the other. \\
	(iv) An arm is \emph{Pareto optimal} given context $x$ if it is not dominated by any other arm given context $x$. Given a particular context $x$, the set of all Pareto optimal arms is called the \emph{Pareto front}, and is denoted by ${\cal O}(x)$.
\end{defn}

\rr{In the following remark, we explain the connection between lexicographic optimality and Pareto optimality.}
\begin{remark}\label{remark:poptimal}
Note that $a^*(x) \in {\cal O}(x)$ for all $x \in {\cal X}$ since $a^*(x)$ is not dominated by any other arm. For all $a \in {\cal A}$, we have $\mu^1_{*}(x) \geq \mu^1_a(x)$. By definition of $a^*(x)$ if there exists an arm $a$ for which $\mu^2_a(x) > \mu^2_{*}(x)$, then we must have $\mu^1_a(x) < \mu^1_{*}(x)$. Such an arm will be incomparable with $a^*(x)$.
\end{remark}

\subsection{Definitions of the 2D Regret and the Pareto Regret}

Initially, the learner does not know the expected rewards; it learns them over time. The goal of the learner is to compete with an oracle, which knows the expected rewards of the arms for every context and chooses the optimal arm given the current context. 
Hence, the 2D regret of the learner by round $T$ is defined as the tuple $(\text{Reg}^1(T),\text{Reg}^2(T))$, where
\begin{align}
\text{Reg}^i(T) 
:= \sum_{t=1}^{T}   \mu^{i}_{*}(x_t)
- \sum_{t=1}^{T} \mu^{i}_{a_t}(x_t) , ~ i \in \{ 1, 2 \}  \label{eqn:regretdefinition}
\end{align}
for an arbitrary sequence of contexts $x_1,\ldots,x_T$. When $\text{Reg}^1(T) = O(T^{\gamma_1})$ and $\text{Reg}^2(T) = O(T^{\gamma_2})$ we say that the 2D regret is $O(T^{\max(\gamma_1, \gamma_2)})$.

Another interesting performance measure is the Pareto regret \cite{drugan2013designing}, which measures the loss of the learner with respect to arms in the Pareto front. To define the Pareto regret, we first define the Pareto suboptimality gap (PSG).
\begin{defn}[PSG of an arm]
	The PSG of an arm $a \in {\cal A}$ given context $x$, denoted by $\Delta_{a}(x)$, is defined as the minimum scalar $\epsilon \geq 0$ that needs to be added to all entries of $\bs{\mu}_{a}(x)$ such that $a$ becomes a member of the Pareto front. Formally, 
	\begin{align*}
	\Delta_{a}(x) := \inf_{\epsilon \geq 0} \epsilon ~~ \text{s.t.} ~~  (\bs{\mu}_{a}(x) + \bs{\epsilon}) \> || \> \bs{\mu}_{a'}(x), \forall a' \in {\cal O}(x) 
	\end{align*}
	where $\bs{\epsilon}$ is a $2$-dimensional vector, whose entries are $\epsilon$.
\end{defn}
Based on the above definition, the Pareto regret of the learner by round $T$ is given by
\begin{align}
\text{PR}(T) := \sum_{t=1}^T \Delta_{a_t}(x_t) \label{eqn:PRregret} . 
\end{align}

\rr{Our goal is to design a learning algorithm whose 2D and Pareto regrets are sublinear functions of $T$ with high probability. This ensures that the average regrets diminish as $T \rightarrow \infty$, and hence, enables the learner to perform on par with an oracle that always selects the optimal arms in terms of the average reward.}

\subsection{Applications of CMAB-DO}\label{sec:applications}

In this subsection we describe four possible applications of CMAB-DO. 

\subsubsection{Multichannel Communication} 
Consider a multichannel communication application in which a user chooses a channel $Q \in  {\cal Q}$ and a transmission rate $R \in {\cal R}$ in each round after receiving context $x_t := \{ \text{SNR}_{Q,t} \}_{Q \in {\cal Q}}$, where $\text{SNR}_{Q,t}$ is the transmit signal to noise ratio of channel $Q$ in round $t$.  For instance, if each channel is also allocated to a primary user, then $\text{SNR}_{Q,t}$ can change from round to round due to time varying transmit power constraint in order not to cause outage to the primary user on channel $Q$.

In this setup, each arm corresponds to a transmission rate-channel pair $(R,Q)$ denoted by $a_{R,Q}$. Hence, the set of arms is ${\cal A} =  {\cal R} \times {\cal Q}$.
When the user completes its transmission at the end of round $t$, it receives a $2$-dimensional reward where the dominant one is related to throughput and the non-dominant one is related to reliability.
Here, $r^2_t \in \{0,1\}$ where $0$ and $1$ correspond to failed and successful transmission, respectively. Moreover, the success rate of $a_{R,Q}$ is equal to $\mu^2_{a_{R,Q}}(x_t) = 1- p_{\text{out}}(R,Q, x_t)$, where $p_{\text{out}}(\cdot)$ denotes the outage probability. Here, $p_{\text{out}}(R,Q, x_t)$ also depends on the gain on channel $Q$ whose distribution is unknown to the user. 
On the other hand, for $a_{R,Q}$, $r^1_t \in \{0,R/R_{\max}\}$ and $\mu^1_{a_{R,Q}}(x_t) = R(1 - p_{\text{out}}(R,Q, x_t)) / R_{\max}$, where $R_{\max}$ is the maximum rate. It is usually the case that the outage probability increases with $R$, so maximizing the throughput and reliability are usually conflicting objectives.\footnote{Note that in this example, given that arm $a_{R,Q}$ is selected, we have $\kappa^1_t = r^1_t - \mu^1_{a_{R,Q}}(x_t)$ and $\kappa^2_t = r^2_t - \mu^2_{a_{R,Q}}(x_t)$. Clearly, both $\kappa^1_t$ and $\kappa^2_t$ are zero mean with support in $[-1,1]$. Hence, they are $1$-sub-Gaussian.} \rr{Illustrative results on this application are given in Section \ref{sec:simchannel}.}

\subsubsection{Online Binary Classification} Consider a medical diagnosis problem where a patient with context $x_t$ (including features such as age, gender, medical test results etc.) arrives in round $t$. Then, this patient is assigned to one of the experts in ${\cal A}$ who will diagnose the patient.  In reality, these experts can either be clinical decision support systems or humans, but the classification performance of these experts are context dependent and unknown a priori.
In this problem, the dominant objective can correspond to accuracy while the non-dominant objective can correspond to false negative rate. For this case, the rewards in both objectives are binary, and depend on whether the classification is correct and a positive case is correctly identified.

\subsubsection{Recommender System}

Recommender systems involve optimization of multiple metrics like novelty and diversity in addition to accuracy \cite{zhou2010solving, konstan2006lessons}. Below, we describe how a recommender system with accuracy and diversity metrics can be modeled using CMAB-DO.

At the beginning of round $t$ a user with context $x_t$ arrives to the recommender system. Then, an item from set ${\cal A}$ is recommended to the user along with a novelty rating box which the user can use to rate the item as novel or not novel.\footnote{An example recommender system that uses this kind of feedback is given in \cite{konstan2006lessons}.} The recommendation is considered to be accurate when the user clicks to the item, and is considered to be novel when the user rates the item as novel.\footnote{In reality, it is possible that some users may not provide the novelty rating. These users can be discarded from the calculation of the regret.} Thus, $r^1_t=1$ if the user clicks to the item and $0$ otherwise. Similarly, $r^2_t = 1$ if the user rates the item as novel and $0$ otherwise. The distribution of $(r^1_t,r^2_t)$ depends on $x_t$ and is unknown to the recommender system.

Another closely related application is display advertising \cite{mccoy2007effects}, where an advertiser can place an ad to the publisher's website for the user currently visiting the website through a payment mechanism. The goal of the advertiser is to maximize its click through rate while keeping the costs incurred through payments at a low level. Thus, it aims at placing an ad only when the current user with context $x_t$ has positive probability of clicking to the ad. Illustrative results on this application are given in Section \ref{sec:recsys}.

\subsubsection{Network Routing}

Packet routing in a communication network commonly involves multiple paths. Adaptive packet routing can improve the performance by avoiding congested and faulty links. In many networking problems, it is desirable to minimize energy consumption as well as the delay due to the energy constraints of sensor nodes. For instance, lexicographic optimality is used in \cite{shah2009lexicographically} to obtain routing flows in a wireless sensor network with energy limited nodes. Moreover, \cite{li2011unified} studies a communication network with elastic and inelastic flows, and proposes load-balancing and rate-control algorithms that prioritize satisfying the rate demanded by inelastic traffic.  

Given a source destination pair $(src,dst)$ in an energy constrained wireless sensor network, we can formulate routing of the flow from node $src$ to node $dst$ using CMAB-DO. At the beginning of each round, the network manager observes the network state $x_t$, which can be the normalized round-trip time on some measurement paths. Then, it selects a path from the set of available paths ${\cal A}$ and observes the normalized random energy consumption $c^1_t$ and delay $c^2_t$ over the selected path. These costs are converted to rewards by setting $r^1_t = 1 - c^1_t$ and $r^2_t = 1 - c^2_t$.

\section{The Learning Algorithm} \label{sec:alg}

We introduce MOC-MAB in this section. Its pseudocode is given in Algorithm \ref{algorithm:MOCMAB}. 

MOC-MAB uniformly partitions ${\cal X}$ into $m^d$ hypercubes with edge lengths $1/m$. This partition is denoted by ${\cal P}$. For each $p \in {\cal P}$ and $a \in {\cal A}$ it keeps: 
(i) a counter $N_{a,p}$ that counts the number of times the context was in $p$ and arm $a$ was selected before the current round, 
(ii) the sample mean of the rewards obtained from rounds prior to the current round in which the context was in $p$ and arm $a$ was selected, i.e., $\hat{\mu}^1_{a,p}$ and $\hat{\mu}^2_{a,p}$ for the dominant and non-dominant objectives, respectively.
The idea behind partitioning is to utilize the similarity of arm rewards given in Assumption \ref{as:1} to learn together for groups of similar contexts. Basically, when the number of sets in the partition is small, the number \rr{of} past samples that fall into a specific set is large; however, the similarity of the past samples that fall into the same set is small. The optimal partitioning should balance the inaccuracy in arm reward estimates that results form these two conflicting facts. 

\begin{algorithm}
\caption{MOC-MAB}\label{algorithm:MOCMAB}
\begin{algorithmic}[1] 
\STATE Input: $T$, $d$, $L$, $\alpha$, $m$, $\beta$
\STATE Initialize sets: Create partition ${\cal P}$ of ${\cal X}$ into
$m^d$ identical hypercubes
\STATE Initialize counters: $N_{a,p}=0$, $\forall a \in {\cal A}$, $\forall p \in {\cal P}$, $t=1$
\STATE Initialize estimates: $\hat{\mu}_{a,p}^1 = \hat{\mu}_{a,p}^2 = 0$, $\forall a \in {\cal A}$, $\forall p \in {\cal P}$
\WHILE{$1 \leq t \leq T$}
\STATE Find $p^* \in {\cal P}$ such that $x_t \in p^*$
\STATE Compute $g_{a,p^*}^i$ for $a \in{\cal A}$, $i \in \{1,2\}$ as given in \eqref{eqn:indices}
\STATE Set $a^*_1 = \argmax_{a \in {\cal A}} g_{a,p^*}^1$ (break ties randomly)
\IF{$u_{a^*_1,p^*} > \beta v$}
\STATE{Select arm $a_t = a^*_1$}
\ELSE
\STATE{Find set of candidate optimal arms $\hat{{\cal A}}^*$ as given in \eqref{eqn:candidates}}
\STATE{Select arm $a_t = \argmax_{a \in \hat{{\cal A}}^*} g^2_{a,p^*}$ (break ties randomly)}
\ENDIF
\STATE{Observe $\bs{r}_t = (r^1_t, r^2_t)$}
\STATE{$\hat{\mu}_{a_t,p^*}^i \leftarrow 
( \hat{\mu}_{a_t,p^*}^i N_{a_t,p^*} + r^i_t ) /
( N_{a_t,p^*} + 1 )$, $i \in \{1,2\}$}
\STATE{$N_{a_t,p^*} \leftarrow N_{a_t,p^*} + 1$} 
\STATE{$t \leftarrow t+1$} 
\ENDWHILE
\end{algorithmic}
\end{algorithm}

At round $t$, MOC-MAB first identifies the hypercube in ${\cal P}$ that contains $x_t$, which is denoted by $p^*$.\footnote{If the context arrives to the boundary of multiple hypercubes, then it is randomly assigned to one of them.} Then, it calculates the following indices for the rewards in the dominant and the non-dominant objectives: 
\begin{align}
g_{a,p^*}^i := \hat{\mu}_{a,p^*}^i + u_{a,p^*}, ~ i \in \{1,2\} 
\label{eqn:indices}
\end{align}
where the {\em uncertainty level} $u_{a,p} := \sqrt{2 A_{m,T} /N_{a,p}}$, $A_{m,T} := (1+2\log(4|{\cal A}|m^{d}T^{3/2}))$ represents the uncertainty over the sample mean estimate of the reward due to the number of instances that are used to compute $\hat{\mu}_{a,p^*}^i$.\footnote{Although MOC-MAB requires $T$ as input, it can run without the knowledge of $T$ beforehand by applying a method called the doubling-trick. See \cite{cesa1997use} and \cite{tekin2016confidence} for a discussion on the doubling-trick.}
Hence, a UCB for $\mu^i_a(x)$ is $g_{a,p}^i + v$ for $x \in p$, where $v := L d^{\alpha/2} m^{-\alpha}$ denotes the non-vanishing uncertainty term due to context set partitioning. Since this term is non-vanishing, we also name it the {\em margin of tolerance}. 
%
%
The main learning principle in such a setting is called optimism under the face of uncertainty. The idea is to inflate the reward estimates from arms that are not selected often by a certain level, such that the inflated reward estimate becomes an upper confidence bound for the true expected reward with a very high probability. This way, arms that are not selected frequently are explored, and this exploration potentially helps the learner to discover arms that are better than the arm with the highest estimated reward. 
As expected, the uncertainty level vanishes as an arm gets selected more often. 

After calculating the UCBs, MOC-MAB judiciously determines the arm to select based on these UCBs. It is important to note that the choice $a^*_1 := \argmax_{a \in {\cal A}} g_{a,p^*}^1$ can be highly suboptimal for the non-dominant objective. To see this, consider a very simple setting, where ${\cal A} = \{a,b\}$, $\mu^1_a(x) = \mu^1_b(x) = 0.5$, $\mu^2_a(x)=1$ and $\mu^2_b(x)=0$ for all $x \in {\cal X}$. 
For an algorithm that always selects $a_t = a^*_1$ and that randomly chooses one of the arms with the highest index in the dominant objective in case of a tie, both arms will be equally selected in expectation. Hence, due to the noisy rewards, there are sample paths in which arm $2$ is selected more than half of the time. For these sample paths, the expected regret in the non-dominant objective is at least $T/2$. 
MOC-MAB overcomes the effect of the noise mentioned above due to the randomness in the rewards and the partitioning of ${\cal X}$ by creating a safety margin below the maximal index $g_{a^*_1,p^*}^1$ for the dominant objective, when its confidence for $a^*_1$ is high, i.e., when $u_{a^*_1,p^*} \leq \beta v$, where $\beta>0$ is a constant. 
For this, it calculates the set of candidate optimal arms given as 
\begin{align}
\hat{{\cal A}}^* &:= \left\{ a \in {\cal A}: 
g^1_{a,p*} \geq \hat{\mu}_{a^*_1,p^*}^1 - u_{a^*_1,p^*} - 2 v \right\}  \label{eqn:candidates} \\
&= 
\left\{ a \in {\cal A}: 
\hat{\mu}_{a,p*}^1 \geq \hat{\mu}_{a^*_1,p^*}^1 - u_{a^*_1,p^*} - u_{a,p^*} - 2 v \right\} . \notag 
\end{align}
Here, the term $- u_{a^*_1,p^*} - u_{a,p^*} - 2 v$ accounts for the joint uncertainty over the sample mean rewards of arms $a$ and $a^*_1$. Then, MOC-MAB selects $a_t = \argmax_{a \in \hat{{\cal A}}^*} g^2_{a,p^*}$. 

On the other hand, when its confidence for $a^*_1$ is low, i.e., when $u_{a^*_1,p^*} > \beta v$, it has a little hope even in selecting an optimal arm for the dominant objective. In this case it just selects $a_t = a^*_1$ to improve its confidence for $a^*_1$. After its arm selection, it receives the random reward vector $\bs{r}_t$, which is then used to update the counters and the sample mean rewards for $p^*$. 

\begin{remark}
At each round, finding the set in ${\cal P}$ that $x_t$ belongs to requires $O(d)$ computations.
Moreover, each of the following processes requires $O(|{\cal A}|)$ computations:
(i) finding maximum value among the indices of the dominant objective, (ii) creating a candidate set and finding maximum value among the indices of the non-dominant objective. Hence, MOC-MAB requires $O(d T) + O(|{\cal A}|T)$ computations in $T$ rounds. In addition, the memory complexity of MOC-MAB is \rr{$O(m^d |{\cal A}|)$}.
\end{remark}

\begin{remark}
MOC-MAB allows the sample mean reward of the selected arm to be less than the sample mean reward of $a^*_1$ by at most $u_{a^*_1,p^*} + u_{a,p^*} + 2v$. Here, $2v$ term does not vanish as arms get selected since it results from the partitioning of the context set. While setting $v$ based on the time horizon allows the learner to control the regret due to partitioning, in some settings having this non-vanishing term allows MOC-MAB to achieve reward that is much higher than the reward of the oracle in the non-dominant objective. Such an example is given in Section \ref{sec:recsys}.
\end{remark}

\section{Regret Analysis}\label{sec:regA}

In this section we prove that both the 2D regret and the Pareto regret of MOC-MAB are sublinear functions of $T$. Hence, MOC-MAB is average reward optimal in both regrets. First, we introduce the following as preliminaries.

For an event ${\cal F}$, let ${\cal F}^c$ denote the complement of that event.
For all the parameters defined in Section \ref{sec:alg}, we explicitly use the round index $t$, when referring to the value of that parameter at the beginning of round $t$. For instance, $N_{a,p}(t)$ denotes the value of $N_{a,p}$ at the beginning of round $t$. 
Let $N_p(t)$ denote the number of context arrivals to $p \in {\cal P}$ by the end of round $t$, $\tau_p(t)$ denote the round in which a context arrives to $p \in {\cal P}$ for the $t$th time, and $R^i_{a}(t)$ denote the random reward of arm $a$ in objective $i$ in round $t$.
Let \rev{$\tilde{x}_{p}(t) := x_{\tau_p(t)}$}, $\tilde{R}^i_{a,p}(t) := R^i_{a}(\tau_p(t))$, $\tilde{N}_{a,p}(t) := N_{a,p}(\tau_p(t))$, 
$\tilde{\mu}^i_{a,p}(t) := \hat{\mu}^i_{a,p}(\tau_p(t))$, 
$\tilde{a}_{p}(t) := a_{\tau_p(t)}$, \rev{$\tilde{\kappa}^i_p(t) := \kappa^i_{\tau_p(t)}$} and $\tilde{u}_{a,p}(t) := u_{a,p}(\tau_p(t))$.
Let ${\cal T}_p := \{ t \in \{1,\ldots,T\} : x_t \in p \}$ denote the set of rounds for which the context is in $p \in {\cal P}$.

Next, we define the following lower and upper bounds: $L^i_{a,p}(t) := \tilde{\mu}^i_{a,p}(t) - \tilde{u}_{a,p}(t)$ and $U^i_{a,p}(t) := \tilde{\mu}^i_{a,p}(t) + \tilde{u}_{a,p}(t)$ for $i \in \{1,2\}$. 
Let 
\begin{align*}
\text{UC}^i_{a,p} := 
\bigcup_{t=1}^{N_{p}(T)} \{ \mu^i_a(\tilde{x}_{p}(t)) \notin
[ L^i_{a,p}(t) -v , U^i_{a,p}(t) +v ] \}
\end{align*}
denote the event that the learner is not confident about its reward estimate in objective $i$ for at least once in rounds in which the context is in $p$ by time $T$. Here $L^i_{a,p}(t) -v$ and $U^i_{a,p}(t) +v$ are the lower confidence bound (LCB) and UCB for $\mu^i_a(\tilde{x}_{p}(t))$, respectively.
Also, let $\text{UC}^i_{p} := \cup_{a \in {\cal A}} \text{UC}^i_{a,p}$, $\text{UC}_{p} := \cup_{i \in \{1,2 \}} \text{UC}^i_{p}$ and $\text{UC} := \cup_{p \in {\cal P}} \text{UC}_{p}$, and for each $i \in \{1,2\}$, $p \in {\cal P}$ and $a \in {\cal A}$, let 
\begin{align*}
\overline{\mu}^i_{a,p} &= \sup_{x \in p} \mu^i_a(x) ~\text{ and }~
\underline{\mu}^i_{a,p} = \inf_{x \in p} \mu^i_a(x) .
\end{align*}

Let 
\begin{align*}
\text{Reg}^i_p(T) := \ \sum_{t=1}^{N_p(T)} \mu^i_{*}(\tilde{x}_p(t)) 
- \sum_{t=1}^{N_p(T)} \mu^i_{\tilde{a}_p(t)}(\tilde{x}_p(t)) 
\end{align*}
denote the regret incurred in objective $i$ for rounds in ${\cal T}_p$ (regret incurred in $p \in {\cal P}$). Then, the total regret in objective $i$ can be written as
\begin{align}
\text{Reg}^i(T) = \sum_{p \in {\cal P}} \text{Reg}^i_p(T) . \label{eqn:partitiondecompose1}
\end{align}
Thus, the expected regret in objective $i$ becomes 
\begin{align}
\expect {\text{Reg}^i(T) } = \sum_{p \in {\cal P}} \expect{ \text{Reg}^i_p(T) } . \label{eqn:partitiondecompose2}
\end{align}
In the following analysis, we will bound both $\text{Reg}^i(T)$ under the event $\text{UC}^c$ and $\expect {\text{Reg}^i(T) }$. For the latter, we will use the following decomposition:
\begin{align}
&\expect{ \text{Reg}^i_{p}(T) } \notag \\
&= \expect{ \text{Reg}^i_{p}(T) | \text{UC} } 
\Pr ( \text{UC}  )  
+ \expect{ \text{Reg}^i_{p}(T) | \text{UC}^c } \Pr ( \text{UC}^c )  \notag \\
&\leq C^i_{\max} N_{p}(T) \Pr ( \text{UC}  )  
+ \expect{ \text{Reg}^i_{p}(T) | \text{UC}^c }  \label{eqn:partitiondecompose2}
\end{align}
where $C^i_{\max}$ is the maximum difference in the expected reward of an optimal arm and any other arm for objective $i$. 

Having obtained the decomposition in \eqref{eqn:partitiondecompose2}, we proceed by bounding the terms in \eqref{eqn:partitiondecompose2}. For this, we first bound $\Pr ( \text{UC}_p )$ in the next lemma.
\begin{lemma} \label{lemma:prUC}
For any $p \in {\cal P}$, we have $\Pr ( \text{UC}_p ) \leq 1 / ( m^d T )$.
\end{lemma}
\begin{proof}
The proof is given in Appendix A.
\end{proof}
Using the result of Lemma \ref{lemma:prUC}, we obtain
\begin{align}
\Pr (  \text{UC} ) \leq 1/ T \text{ and } \Pr (  \text{UC}^c ) \geq 1 - 1/ T . \label{eqn:UCbound}
\end{align}
To prove the lemma above, we use the concentration inequality given in Lemma 6 in \cite{abbasi2011improved} to bound the probability of $\text{UC}^i_{a,p}$. However, a direct application of this inequality is not possible to our problem, due to the fact that the context sequence $\tilde{x}_p(1), \ldots,\tilde{x}_p(N_p(t))$ does not have identical elements, which makes the mean values of $\tilde{R}^i_{a,p}(1),\ldots,\tilde{R}^i_{a,p}(N_p(t))$ different. 
To overcome this problem, we use the sandwich technique proposed in \cite{tekin2016confidence} in order to bound the rewards sampled from actual context arrivals between the rewards sampled from two specific processes that are related to the original process, where each process has a fixed mean value.

After bounding the probability of the event $\Pr ( \text{UC}_p )$, we bound the instantaneous \rev{(single round)} regret on event $\Pr ( \text{UC}^c )$. For simplicity of notation, in the following lemmas we use \rev{$a^*(t) := a^*(\tilde{x}_p(t))$} to denote the optimal arm, \rev{$\tilde{a}(t) := \tilde{a}_{p}(t)$} to denote the arm selected at round \rev{$\tau_p(t)$} and \rev{$\hat{a}^*_1(t)$} to denote the arm whose first index is highest at round \rev{$\tau_p(t)$, when the set $p \in {\cal P}$ that the context belongs to is obvious.}

\rev{The following lemma shows that on event $\text{UC}^c_p$ the regret incurred in a round $\tau_p(t)$ for the dominant objective can be bounded as function of the difference between the upper and lower confidence bounds plus the margin of tolerance.}

\begin{lemma} \label{lemma:indexdiff}
When MOC-MAB is run, on event $\text{UC}^c_p$, we have
\begin{align}
\mu^1_{a^*(t)} ( \tilde{x}_{p}(t)  )  -   \mu^1_{\tilde{a}(t)} ( \tilde{x}_{p}(t)  )
\leq&  U^1_{\tilde{a}(t),p}(t) - L^1_{ \tilde{a}(t),p}(t) \notag \\ 
&+ 2(\beta + 2) v \notag
\end{align}
for all  $t \in \{1,\ldots,N_{p}(T)\}$.
\end{lemma} 
\begin{proof}
We consider two cases. 
When $\tilde{u}_{\hat{a}^*_1(t),p}(t)  \leq \beta v$, we have
\begin{align}
U^1_{\tilde{a}(t),p}(t)  &\geq  L^1_{\hat{a}^*_1(t),p}(t) -2v \notag \\
&\geq U^1_{\hat{a}^*_1(t),p}(t) -2\tilde{u}_{\hat{a}^*_1(t),p}(t) -2v  \notag \\
&\geq U^1_{\hat{a}^*_1(t),p}(t) -  2 ( \beta + 1) v . \notag
\end{align}
On the other hand, when $ \tilde{u}_{\hat{a}^*_1(t),p}(t) > \beta v$, the selected arm is \rev{$\tilde{a}(t) = \hat{a}^*_1(t)$}. Hence, we obtain
\begin{align}
U^1_{\tilde{a}(t),p}(t)  =  U^1_{\hat{a}^*_1(t),p}(t)  \geq U^1_{\hat{a}^*_1(t),p}(t) -  2 ( \beta + 1) v . \notag 
\end{align}
Thus, for both cases, we have 
\begin{align}
U^1_{\tilde{a}(t),p}(t) \geq U^1_{\hat{a}^*_1(t),p}(t) -  2 ( \beta + 1) v    \label{eqn:multi1}
\end{align}
and 
\begin{align}
U^1_{\hat{a}^*_1(t),p}(t) \geq U^1_{ a^*(t),p}(t)  .    \label{eqn:multi2}
\end{align}

On event $\text{UC}^c_p$, we also have
\begin{align}
\mu^1_{a^*(t)} ( \tilde{x}_{p}(t)  ) \leq U^1_{ a^*(t),p}(t) + v    \label{eqn:multi3}
\end{align}
and 
\begin{align}
\mu^1_{\tilde{a}(t)} ( \tilde{x}_{p}(t)  ) \geq L^1_{\tilde{a}(t),p}(t) - v      \label{eqn:multi4} .
\end{align}
By combining \eqref{eqn:multi1}-\eqref{eqn:multi4}, we obtain
\begin{align}
\mu^1_{a^*(t)} ( \tilde{x}_{p}(t)  )  -   \mu^1_{\tilde{a}(t)} ( \tilde{x}_{p}(t)  )
\leq&  U^1_{\tilde{a}(t),p}(t) - L^1_{\tilde{a}(t),p}(t)  \notag \\
&+ 2(\beta + 2) v . \notag
\end{align}
\end{proof}

\rev{The lemma below bounds the regret incurred in a round $\tau_p(t)$ for the non-dominant objective on event $\text{UC}^c_p$ when the uncertainty level of the arm with the highest index in the dominant objective is low.}

\begin{lemma} \label{lemma:multi2}
When MOC-MAB is run, on event $\text{UC}^c_p$, for $t \in \{1,\ldots,N_{p}(T)\}$ if
\begin{align}
\tilde{u}_{\hat{a}^*_1(t),p}(t) \leq \beta v      \notag
\end{align}
holds, then we have
\begin{align}
\mu^2_{a^*(t)} ( \tilde{x}_{p}(t) ) -  \mu^2_{\tilde{a}(t)} ( \tilde{x}_{p}(t) ) \leq 
U^2_{\tilde{a}(t),p}(t)  -  L^2_{ \tilde{a}(t),p} + 2v . \notag
\end{align}
\end{lemma} 
\begin{proof}
When 
$\tilde{u}_{\hat{a}^*_1(t),p}(t)  \leq \beta v$ holds,
all arms that are selected as candidate optimal arms have their index for objective $1$ in the interval 
$[  L^1_{\hat{a}^*_1(t),p}(t) - 2v ,   U^1_{\hat{a}^*_1(t),p}(t) ]$. Next, we show that $U^1_{ a^*(t),p}(t)$
is also in this interval. 

On event $\text{UC}^c_p$, we have
\begin{align}
\mu^1_{a^*(t)} ( \tilde{x}_{p}(t) ) & \in [ L^1_{a^*(t),p}(t) - v,   U^1_{a^*(t),p}(t) + v  ]      \notag \\
\mu^1_{\hat{a}^*_1(t)} ( \tilde{x}_{p}(t) ) & \in [ L^1_{\hat{a}^*_1(t),p}(t) - v,   U^1_{\hat{a}^*_1(t),p}(t) + v  ]  .   \notag
\end{align}
We also know that 
\begin{align}
\mu^1_{a^*(t)} ( \tilde{x}_{p}(t) ) \geq \mu^1_{\hat{a}^*_1(t)} ( \tilde{x}_{p}(t) ) .  \notag
\end{align}
Using the inequalities above, we obtain
\begin{align}
U^1_{a^*(t),p}(t) \geq  \mu^1_{a^*(t)} ( \tilde{x}_{p}(t) ) - v 
&\geq \mu^1_{\hat{a}^*_1(t)} ( \tilde{x}_{p}(t) )  - v  \notag \\ 
&\geq L^1_{\hat{a}^*_1(t),p}(t) - 2v .      \notag
\end{align}

\rev{Since the selected arm has the maximum index for the non-dominant objective among all arms whose indices for the dominant objective are in $[  L^1_{\hat{a}^*_1(t),p}(t) - 2v ,   U^1_{\hat{a}^*_1(t),p}(t) ]$, we have $U^2_{ \tilde{a}(t),p}(t) \geq U^2_{ a^*(t),p}(t)$.}
Combining this with the fact that $\text{UC}^c_p$ holds, we get
\begin{align}
\mu^2_{\tilde{a}(t)} ( \tilde{x}_{p}(t) ) \geq  L^2_{ \tilde{a}(t),p}(t) - v  \label{eqn:lemma31}
\end{align}
and
\begin{align}
\mu^2_{a^*(t)} ( \tilde{x}_{p}(t) ) \leq U^2_{ a^*(t),p}(t) + v \leq U^2_{ \tilde{a}(t),p}(t) + v .   \label{eqn:lemma32}
\end{align}
Finally, by combining \eqref{eqn:lemma31} and \eqref{eqn:lemma32}, we obtain
\begin{align}
\mu^2_{a^*(t)} ( \tilde{x}_{p}(t) ) -  \mu^2_{\tilde{a}(t)} ( \tilde{x}_{p}(t) ) \leq 
U^2_{ \tilde{a}(t),p}(t)  -  L^2_{ \tilde{a}(t),p}(t) + 2v . \notag
\end{align}
\end{proof}

\rev{For any $p \in {\cal P}$, we also need to bound the regret of the non-dominant objective for rounds in which $\tilde{u}_{\hat{a}^*_1(t),p}(t) > \beta v$, $t \in \{1,\ldots,N_p(T)\}$.} 
\begin{lemma} \label{lemma:multi3}
When MOC-MAB is run, the \rev{number of rounds in ${\cal T}_p$ for which}
$\tilde{u}_{\hat{a}^*_1(t),p}(t)  > \beta v$ happens is bounded above by
\begin{align}
|{\cal A}| \left( \frac{2 A_{m,T} } {\beta^2 v^2} + 1 \right). \notag
\end{align}
\end{lemma} 
\begin{proof}
This event happens when $\tilde{N}_{\hat{a}^*_1(t),p}(t) < 2 A_{m,T} /(\beta^2 v^2)$. Every such event will result in an increase in the value of $N_{\hat{a}^*_1(t),p}$ by one. 
Hence, for $p \in {\cal P}$ and $a \in {\cal A}$, the number of times
$\tilde{u}_{a,p}(t)  > \beta v$ can happen is bounded above by
$2 A_{m,T} / (\beta^2 v^2) + 1 $. The final result is obtained by summing over all arms. 
\end{proof}

\rev{In the next lemmas, we bound $\text{Reg}^1_{p}(t)$ and $\text{Reg}^2_{p}(t)$ given that $\text{UC}^c$ holds.}

\begin{lemma} \label{lemma:instreg1d}
When MOC-MAB is run, on event $\text{UC}^c$, we have for all $p \in {\cal P}$
\begin{align}
 \text{Reg}^1_{p}(t) & \leq 
|{\cal A} |C^1_{\max} + 2 B_{m,T} \sqrt{   |{\cal A}|  N_{p}(t)   } + 2 (\beta + 2) v N_{p}(t) . \notag
\end{align}
where $B_{m,T} := 2 \sqrt{2 A_{m,T}}$.
\end{lemma}
\begin{proof}
The proof is given in Appendix B. 
\end{proof}

\begin{lemma} \label{lemma:instreg2n}
When MOC-MAB is run, on event $\text{UC}^c$ we have for all $p \in {\cal P}$
\begin{align}
\text{Reg}^2_{p}(t)  \leq&   C^2_{\max} |{\cal A}| \left( \frac{2 A_{m,T}} {\beta^2 v^2} + 1 \right)  
+  2 v N_{p}(t) \notag \\
&+  2 B_{m,T} \sqrt{   |{\cal A}|  N_{p}(t)   } .  \notag
\end{align}
\end{lemma}
\begin{proof}
The proof is given in Appendix C.
\end{proof}

Next, we use the result of Lemmas \ref{lemma:prUC}, \ref{lemma:instreg1d} and \ref{lemma:instreg2n} to find a bound on  $\text{Reg}^i(t)$ that holds for all $t \leq T$ with probability at least $1-1/T$.  
\begin{theorem} \label{theorem:boundprob}
When MOC-MAB is run, we have for any $i \in \{1,2\}$
\begin{align}
\Pr( \text{Reg}^i (t) < \epsilon_{i}(t) ~ \forall t \in \{ 1,\ldots, T\}  ) \geq 1- 1/T ~\notag
\end{align}
where
\begin{align}
\epsilon_{1}(t) = m^d |{\cal A}|C^1_{\max} + 2 B_{m,T} \sqrt{ |{\cal A}|  m^d t  } + 2 (\beta + 2) v t \notag
\end{align}
and
\begin{align}
\epsilon_{2}(t) =& m^d |{\cal A}|C^2_{\max} + m^d C^2_{\max} |{\cal A}| \left( \frac{2 A_{m,T}} {\beta^2 v^2}  \right)  \notag \\
&+ 2 B_{m,T} \sqrt{ |{\cal A}|  m^d t  } + 2 v t. \notag
\end{align}
\end{theorem}
\begin{proof}
By \eqref{eqn:partitiondecompose1} and Lemmas \ref{lemma:instreg1d} and \ref{lemma:instreg2n}, we have on event $\text{UC}^c$:
\begin{align}
\text{Reg}^1(t)
\leq& m^d |{\cal A}|C^1_{\max} + 2 B_{m,T} \sum_{p \in {\cal P}} \sqrt{ |{\cal A}|  N_{p}(t) } \notag \\
&+ 2 (\beta + 2) v t \notag \\
\leq & m^d |{\cal A}|C^1_{\max} + 2 B_{m,T} \sqrt{ |{\cal A}|  m^d t}   \notag \\
&+ 2 (\beta + 2) v t  \notag
\end{align}
and
\begin{align}
\text{Reg}^2(t)  \leq& m^d |{\cal A}|C^2_{\max} + m^d C^2_{\max} |{\cal A}| \left( \frac{2 A_{m,T}} {\beta^2 v^2}  \right)  \notag \\
& + 2 B_{m,T} \sum_{p \in {\cal P}} \sqrt{ |{\cal A}|  N_{p}(t) }  + 2 v t  \notag \\
\leq& m^d |{\cal A}|C^2_{\max} + m^d C^2_{\max} |{\cal A}| \left( \frac{2 A_{m,T}} {\beta^2 v^2}  \right)  \notag \\
&+ 2 B_{m,T} \sqrt{ |{\cal A}|  m^d t  } + 2 v t \notag
\end{align}
for all $t \leq T$. The result follows from the fact that $\text{UC}^c$ holds with probability at least $1-1/T$.
\end{proof}

The following theorem shows that the expected 2D regret of MOC-MAB by time $T$ is $\tilde{O}(T^{\frac{2\alpha+d}{3\alpha+d}})$.
\begin{theorem} \label{theorem:totalregre1}
When MOC-MAB is run with inputs $m = \lceil T^{1/(3\alpha+d)} \rceil$ and $\beta >0$, we have
\begin{align}
 \expect{ \text{Reg}^1(T) }  \leq& C^1_{\max} +  2^d|{\cal A}| C^1_{\max} T^{\frac{d}{3\alpha+d}} \notag \\
 +& 2 (\beta +2 ) Ld^{\alpha/2} T^{\frac{2\alpha+ d}{3\alpha+d}} 
 \notag \\
 +&  2^{d/2 + 1} B_{m,T} \sqrt{  |{\cal A}|   } T^{ \frac{1.5 \alpha + d}{3 \alpha + d}   } \notag
\end{align}
and
\begin{align}
\expect{ \text{Reg}^2(T) } &\leq  2^{d/2 + 1} B_{m,T} \sqrt{  |{\cal A}|  } T^{ \frac{1.5 \alpha + d}{3 \alpha + d}   } + C^2_{\max}   \notag \\
+& \left( 2Ld^{\alpha/2} + \frac{ C^2_{\max} |{\cal A}| 2^{1+2\alpha+d} A_{m,T}} {\beta^2 L^2 d^\alpha} \right) T^{\frac{2\alpha+ d}{3\alpha+d}} \notag \\
+& 2^d C^2_{\max} |{\cal A}| T^{\frac{d}{3\alpha+d}} . \notag
\end{align}
\end{theorem}
\begin{proof}
$\expect{\text{Reg}^i (T)} $ is bounded by using the result of Theorem \ref{theorem:boundprob} and \eqref{eqn:partitiondecompose2}:
\begin{align}
\expect{ \text{Reg}^i(T) } &\leq  \expect{ \text{Reg}^i(T) | \text{UC}^c } +  \sum_{p \in {\cal P}} C^i_{\max} N_{p}(T) \Pr ( \text{UC}  )   \notag \\
&\leq\expect{ \text{Reg}^i(T) | \text{UC}^c } +  \sum_{p \in {\cal P}} C^i_{\max} N_{p}(T)/T   \notag \\
 &= \expect{ \text{Reg}^i(T) | \text{UC}^c } +  C^i_{\max} . \notag
\end{align} 
Therefore, we have
\begin{align}
\expect{ \text{Reg}^1(T) } &\leq  \epsilon_1(T) + C^1_{\max} \notag \\
\expect{ \text{Reg}^2(T) } &\leq  \epsilon_2(T) + C^2_{\max} . \notag 
\end{align}
It can be shown that when we set $m = \lceil T^{1/(2\alpha+d)} \rceil$ regret bound of the dominant objective becomes $\tilde{O}( T^{(\alpha+d)/(2\alpha+d)} )$ and regret bound of the non-dominant objective becomes $O(T)$. The optimal value for $m$ that makes both regrets sublinear is $m = \lceil T^{1/(3\alpha+d)} \rceil$. With this value of $m$, we obtain
\begin{align}
 \expect{ \text{Reg}^1(T) }  \leq&  2^{d} |{\cal A}| C^1_{\max} T^{\frac{d}{3\alpha+d}} 
 + 2 (\beta +2 ) L d^{\alpha/2} T^{\frac{2\alpha+ d}{3\alpha+d}} 
 \notag \\
 &+  2^{d/2 + 1} B_{m,T} \sqrt{  |{\cal A}|   } T^{ \frac{1.5 \alpha + d}{3 \alpha + d}   } + C^1_{\max} \notag 
\end{align}
and
\begin{align}
\expect{ \text{Reg}^2(T) } \leq&  \left( 2 L d^{\alpha/2}  + \frac{ C^2_{\max} |{\cal A}| 2^{1 + 2\alpha + d} A_{m,T}} {\beta^2 L^2 d^\alpha} \right) T^{\frac{2\alpha+ d}{3\alpha+d}}\notag \\
&+ C^2_{\max}  + 2^d  C^2_{\max} |{\cal A}| T^{\frac{d}{3\alpha+d}} \notag \\
&+  2^{d/2 + 1} B_{m,T} \sqrt{  |{\cal A}|  } T^{ \frac{1.5 \alpha + d}{3 \alpha + d}   } .
  \notag
\end{align}
\end{proof}

From the results above we conclude that both regrets are $\tilde{O} ( T^{(2\alpha+d)/(3\alpha+d)}  )$,
where for the first regret bound the constant that multiplies the highest order of the regret does not depend on ${\cal A}$, while the dependence on this term is linear for the second regret bound. 

Next, we show that the expected value of the Pareto regret of MOC-MAB given in \eqref{eqn:PRregret} is also $\tilde{O} ( T^{(2\alpha+d)/(3\alpha+d)}  )$.
\begin{theorem}\label{thm:pregret}
When MOC-MAB is run with inputs $m = \lceil T^{1/(3\alpha+d)} \rceil$ and $\beta >0$, we have 
\begin{align}
\Pr( \text{PR} (t) < \epsilon_1(t) ~ \forall t \in \{ 1,\ldots, T\}  ) \geq 1- 1/T ~\notag
\end{align}
where $\epsilon_1(t)$ is given in Theorem \ref{theorem:boundprob} and 
\begin{align}
 \expect{ \text{PR}(T) }  \leq& C^1_{\max} +  2^d|{\cal A}| C^1_{\max} T^{\frac{d}{3\alpha+d}} \notag \\
 +& 2 (\beta +2 ) Ld^{\alpha/2} T^{\frac{2\alpha+ d}{3\alpha+d}} 
 \notag \\
 +&  2^{d/2 + 1} B_{m,T} \sqrt{  |{\cal A}|   } T^{ \frac{1.5 \alpha + d}{3 \alpha + d}   } . \notag
\end{align}
\end{theorem}
\begin{proof}
Consider any $p \in {\cal P}$ and $t \in \{1,\ldots,N_p(T)\}$.
By definition $\Delta_{\tilde{a}(t)}(\tilde{x}_p(t)) \leq \mu^1_{a^*(t)} ( \tilde{x}_{p}(t)  )  -   \mu^1_{\tilde{a}(t)} ( \tilde{x}_{p}(t) )$. This holds since for any $\epsilon >0$, adding $\mu^1_{a^*(t)} ( \tilde{x}_{p}(t)  )  -   \mu^1_{\tilde{a}(t)} ( \tilde{x}_{p}(t)  ) + \epsilon$ to $\mu^1_{\tilde{a}(t)} ( \tilde{x}_{p}(t) )$ will either make it (i) dominate the arms in ${\cal O}(\tilde{x}_{p}(t))$ or (ii)  incomparable with the arms in ${\cal O}(\tilde{x}_{p}(t))$. Hence, using the result in Lemma \ref{lemma:indexdiff}, we have on event $\text{UC}^c$
\begin{align}
\Delta_{\tilde{a}(t)}(\tilde{x}_p(t)) \leq  U^1_{\tilde{a}(t),p}(t) - L^1_{ \tilde{a}(t),p}(t) 
+ 2(\beta + 2) v . \notag
\end{align}
Let $\text{PR}_p(T) := \sum_{t=1}^{N_p(T)} \Delta_{\tilde{a}(t)}(\tilde{x}_p(t))$. Hence, $\text{PR}(T) = \sum_{p \in {\cal P}} \text{PR}_p(T)$. Due to this, the results derived for $\text{Reg}^1(t)$ and $\text{Reg}^1(T)$ in Theorems \ref{theorem:boundprob} and \ref{theorem:totalregre1} also hold for $\text{PR}_p(t)$ and $\text{PR}_p(T)$.
\end{proof}

Theorem \ref{thm:pregret} shows that the regret measures $\expect{ \text{Reg}^1(T) }$, $\expect{ \text{Reg}^2(T) }$ and $\expect{ \text{PR}(T) }$ for MOC-MAB are all $\tilde{O} ( T^{(2\alpha+d)/(3\alpha+d)}  )$ when it is run with $m = \lceil T^{1/(3\alpha+d)} \rceil$. This implies that MOC-MAB is average reward optimal in all regret measures as $T \rightarrow \infty$. The growth rate of the Pareto regret can be further decreased by setting $m = \lceil T^{1/(2\alpha+d)} \rceil$. This will make the Pareto regret $\tilde{O} ( T^{(\alpha+d)/(2\alpha+d)} )$ (which matches with the lower bound in \cite{lu2010contextual} for the single-objective contextual MAB with similarity information up to a logaritmic factor) but will also make the regret in the non-dominant objective linear.

%

%

\section{Extensions}\label{sec:extensions}

\subsection{Learning Under Periodically Changing Reward Distributions}\label{sec:learndynamic}

In many practical cases, the reward distribution of an arm changes periodically over time even under the same context. 
For instance, in a recommender system the probability that a user clicks to an ad may change with the time of the day, but the pattern of change can be periodical on a daily basis and this can be known by the system. Moreover, this change is usually gradual over time. In this section, we extend MOC-MAB such that it can deal with such settings.

For this, let $T_s$ denote the period. For the $d$-dimensional context $x_t =(x_{1,t}, x_{2,t}, ..., x_{d,t})$ received at round $t$ let $\hat{x}_t := (x_{1,t}, x_{2,t}, ..., x_{d+1,t})$ denote the extended context where $x_{d+1,t} := ( t \mod T_s )/T_s$ is the time context. Let $\hat{\cal X}$ denote the $d+1$ dimensional extended context set constructed by adding the time dimension to ${\cal X}$. It is assumed that the following holds for the extended contexts.
\begin{assumption}\label{as:2}
Given any $\hat{x}, \hat{x}' \in \hat{{\cal X}}$,
there exists $\hat{L} > 0$ and $0 < \hat{\alpha} \leq 1$ such that for all $i \in \left\{ 1,2 \right\}$ and $a \in {\cal A}$, we have
\begin{align*}
&| \mu^{i}_{a}(\hat{x}) - \mu^{i}_{a}(\hat{x}') |
\leq \hat{L} || \hat{x}- \hat{x}' ||^{\hat{\alpha}} . \notag 
\end{align*}
\end{assumption}
Note that Assumption \ref{as:2} implies Assumption \ref{as:1} with $L = \hat{L}$ and $\alpha = \hat{\alpha}$ when $\hat{x}_{d+1} = \hat{x}'_{d+1}$. Moreover, for two contexts $(x_1,\ldots,x_d,x_{d+1})$ and $(x_1,\ldots,x_d,x'_{d+1})$, we have
\begin{align*}
| \mu^{i}_{a}(\hat{x}) - \mu^{i}_{a}(\hat{x}') |
\leq \hat{L} | x_{d+1} - x'_{d+1} |^{\hat{\alpha}} 
\end{align*}
which implies that the change in the expected rewards is gradual.
Under Assumption \ref{as:2}, the performance of MOC-MAB is bounded as follows.
\begin{corollary}
When MOC-MAB is run with inputs $\hat{L}$, $\hat{\alpha}$, $m = \lceil T^{1/(3\hat{\alpha}+d+1)} \rceil$, and $\beta >0$ by using the extended context set $\hat{{\cal X}}$ instead of the original context set ${\cal X}$, we have 
\begin{align*}
\expect{ \text{Reg}^i(T) } = \tilde{O} ( T^{(2 \hat{\alpha} + d + 1)/(3 \hat{\alpha} + d + 1)}  ) \text{ for } i \in \{1,2\} .
\end{align*}
\end{corollary}
\begin{proof}
The proof simply follows from the proof of Theorem \ref{theorem:totalregre1} by extending the dimension of the context set by one. 
\end{proof}

\subsection{\rr{Lexicographic Optimality for $d_r > 2$ Objectives}}\label{sec:learndynamic}

\rr{
Our problem formulation can be generalized to handle $d_r > 2$ objectives as follows. Let $\bs{r}_t := (r^1_t,\ldots,r^{d_r}_t)$ denote the reward vector in round $t$ and $\bs{\mu}_a(x) := (\mu^1_a(x),\ldots,\mu^{d_r}_a(x))$ denote the expected reward vector for context-arm pair $(x,a)$.
We say that arm $a$ lexicographically dominates arm $a'$ in the first $j$ objectives for context $x$, denoted by $\bs{\mu}_a(x) >_{\text{lex},j} \bs{\mu}_{a'}(x)$ if $\mu^i_a(x) > \mu^i_{a'}(x)$, where $i := \min \{ k \leq j : \mu^k_a(x) \neq  \mu^k_{a'}(x) \}$.\footnote{If $i$ does not exist then $\mu^k_a(x) = \mu^k_{a'}(x)$ for all $k \in \{1,\ldots,j\}$, and hence, arm $a$ does not lexicographically dominate arm $a'$ in the first $j$ objectives.} Then, arm $a$ is defined to be lexicographically optimal for context $x$ if there is no other arm that lexicographically dominates it in $d_r$ objectives. 

Let $\mu^i_{*}(x)$ denote the expected reward of a lexicographically optimal arm for context $x$ in objective $i$. Then, the $d_r$-dimensional regret is defined as follows: 
\begin{align*}
\textbf{Reg}(T) &:= (\text{Reg}^1(T), \ldots, \text{Reg}^{d_r}(T)) \text{ where } \\
\text{Reg}^i(T) &:= \sum_{t=1}^{T}   \mu^{i}_{*}(x_t) - \sum_{t=1}^{T} \mu^{i}_{a_t}(x_t), i \in \{1,\ldots,d_r\} .
\end{align*}
Generalizing MOC-MAB to achieve sublinear regret for all objectives will require construction of a hierarchy of candidate optimal arm sets similar to the one given in \eqref{eqn:candidates}. We leave this interesting research problem as future work, and explain when lexicographically optimality in the first two objectives indicates lexicographic optimality in $d_r$ objectives and why the number of cases in which lexicographically optimality in the first two objectives does not indicate lexicographic optimality in $d_r$ objectives is {\em scarce}.

Let ${\cal A}^*_{j}(x)$ denote the set of lexicographically optimal arms for context $x$ in the first $j$ objectives. We call the case ${\cal A}^*_{2}(x) = {\cal A}^*_{d_r}(x)$ for all $x \in {\cal X}$ the {\em degenerate case} of the $d_r$-objective contextual MAB. Similarly, we call the case when there exists some $x \in {\cal X}$, for which ${\cal A}^*_{2}(x) \neq {\cal A}^*_{d_r}(x)$ as the {\em non-degenerate case} of the $d_r$-objective contextual MAB. Next, we argue that the non-degenerate case is uncommon. Since ${\cal A}^*_{j}(x) \supseteq {\cal A}^*_{j+1}(x)$ for $j \in \{1,\ldots,d_r-1\}$ and there is at least one lexicographically optimal arm, ${\cal A}^*_{2}(x) \neq {\cal A}^*_{d_r}(x)$ implies that ${\cal A}^*_{2}(x)$ is not a singleton. This implies existence of two arms $a$ and $b$ such that $\mu^1_a(x) = \mu^1_b(x)$ and $\mu^2_a(x) = \mu^2_b(x)$. In contrast, for the contextual MAB to be non-trivial, we only require existence of at least one context $x 
\in {\cal X}$ and arms $a$ and $b$ such that $\mu^1_a(x) = \mu^1_b(x)$. 
}

\section{Illustrative Results}\label{sec:sims}

In order to evaluate the performance of MOC-MAB, we run three different experiments both with synthetic and real-world datasets. 

We compare MOC-MAB with the following MAB algorithms:\\
\textbf{Pareto UCB1 (P-UCB1)}: This is the Empirical Pareto UCB1 algorithm proposed in \cite{drugan2013designing}. \\
\textbf{Scalarized UCB1 (S-UCB1)}: This is the Scalarized Multi-objective UCB1 algorithm proposed in \cite{drugan2013designing}. \\
\textbf{Contextual Pareto UCB1 (CP-UCB1)}: This is the contextual version of P-UCB1 which partitions the context set in the same way as MOC-MAB does, and uses a different instance of P-UCB1 in each set of the partition. \\
\textbf{Contextual Scalarized UCB1 (CS-UCB1)}: This is the contextual version of S-UCB1, which partitions the context set in the same way as MOC-MAB does, and uses a different instance of S-UCB1 in each set of the partition. 
\\
\textbf{Contextual Dominant UCB1 (CD-UCB1)}: This is the contextual version of UCB1 \cite{auer2002finite}, which partitions the context set in the same way as MOC-MAB does, and uses a different instance of UCB1 in each set of the partition. This algorithm only uses the rewards from the dominant objective to update the indices of the arms.   \\

For S-UCB1 and CS-UCB1, the weights of the linear scalarization functions are chosen as $[1,0]$, $[0.5,0.5]$ and $[0,1]$. For all contextual algorithms, the partition of the context set is formed by choosing $m$ according to Theorem \ref{theorem:totalregre1}, and $L$ and $\alpha$ are taken as $1$. For MOC-MAB, $\beta$ is chosen as $1$ unless stated otherwise. 
In addition, we scaled down the uncertainty level (also known as the confidence term or the inflation term) of all the algorithms by a constant chosen from 
$\{1, 1/5, 1/10, 1/15, 1/20, 1/25, 1/30\}$, since we observed that the regrets of the algorithms in the dominant objective may become smaller when the uncertainty level is scaled down. The reported results correspond to runs performed using the optimal scale factor for each experiment. 

\subsection{Experiment 1 - Synthetic Dataset}

In this experiment, we compare MOC-MAB with other MAB algorithms on a synthetic multi-objective dataset.
We take ${\cal X} = [0,1]^2$ and assume that the context at each round is chosen uniformly at random from ${\cal X}$. 
We consider $4$ arms and the time horizon is set as $T=10^5$. The expected arm rewards for 3 of the arms are generated as follows: 
We generate 3 multivariate Gaussian distributions for the dominant objective and 3 multivariate Gaussian distributions for the non-dominant objective. 
For the dominant objective, the mean vectors of the first two distributions are set as $[0.3, 0.5]$, and the mean vector of the third distribution is set as $[0.7,0.5]$. Similarly, for the non-dominant objective, the mean vectors of the distributions are set as $[0.3,0.7]$, $[0.3,0.3]$ and $[0.7,0.5]$, respectively. For all the Gaussian distributions the covariance matrix is given by $0.3* \text{I}$ where $\text{I}$ is the $2$ by $2$ identity matrix.
Then, each Gaussian distribution is normalized by multiplying it with a constant, such that its maximum value becomes $1$. These normalized distributions form the expected arm rewards. In addition, the expected reward of the fourth arm for the dominant objective is set as $0$, and its expected reward for the non-dominant objective is set as the normalized multivariate Gaussian distribution with mean vector $[0.7,0.5]$.
We assume that the reward of an arm in an objective given a context $x$ is a Bernoulli random variable whose parameter is equal to the magnitude of the corresponding normalized distribution at context $x$. 

%

\begin{figure}[h!] 
	
	\begin{minipage}[b]{1.0\linewidth}
		\centering
		\centerline{\includegraphics[width=8.5cm]{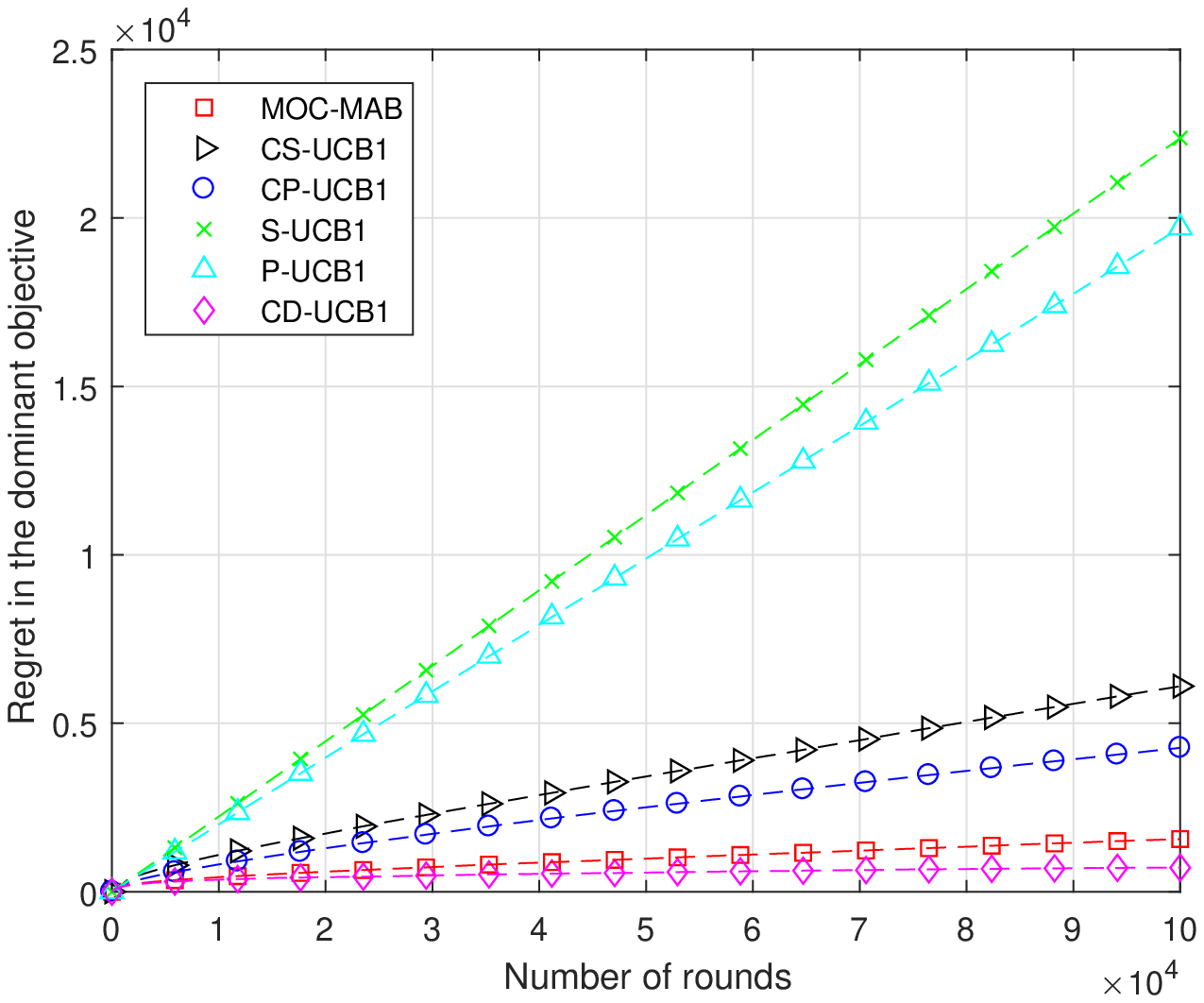}}
		
	\end{minipage}
	
	\begin{minipage}[b]{1.0\linewidth}
		\centering
		\centerline{\includegraphics[width=8.5cm]{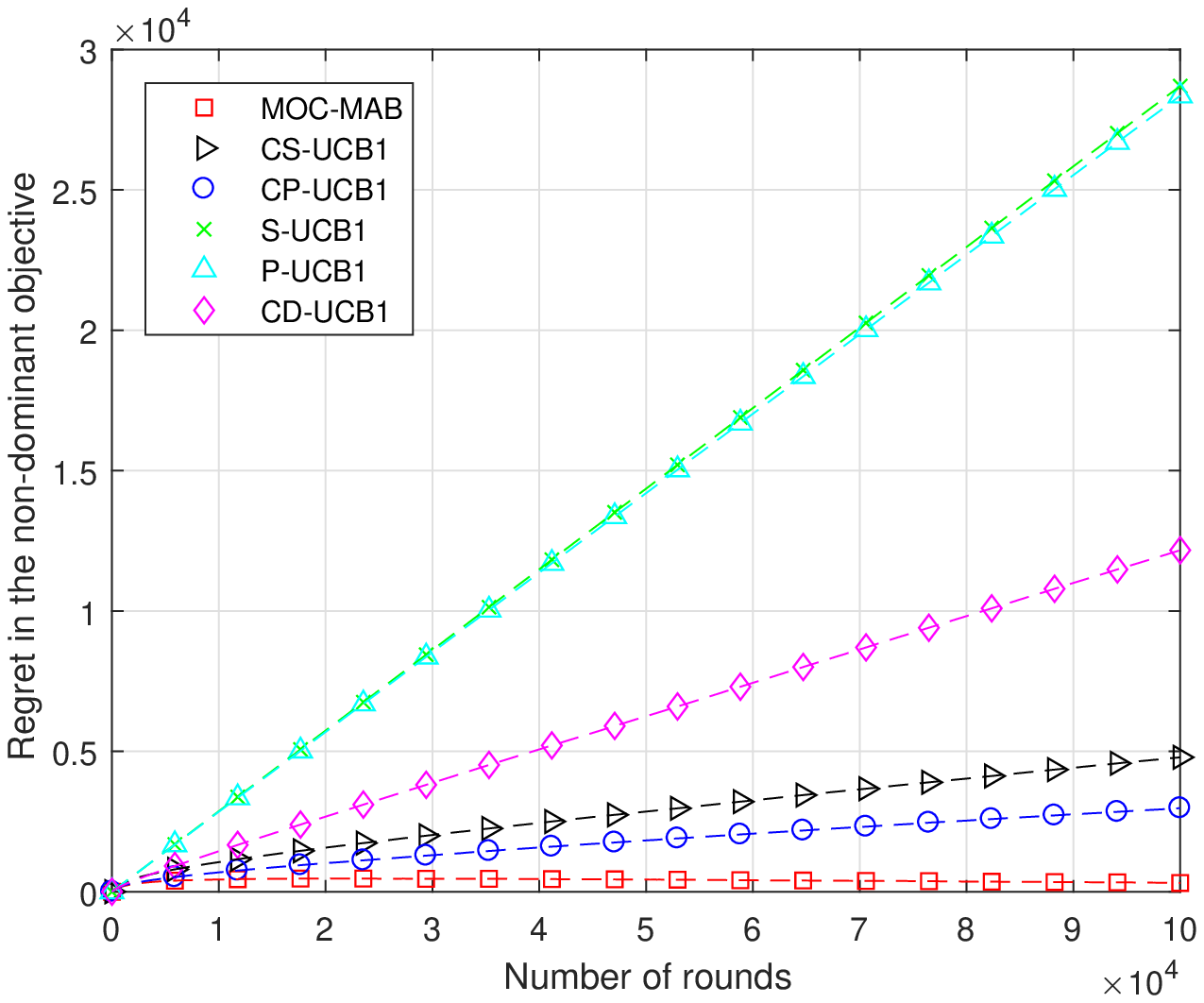}}
		
	\end{minipage}
	
	\caption{Regrets of MOC-MAB and the other algorithms for Experiment 1.}
	\label{fig:p3}
	
\end{figure}

Every algorithm is run $100$ times and the results are averaged over these runs. Simulation results given in Fig. \ref{fig:p3}  show the change in the regret of the algorithms in both objectives as a function of time (rounds). 
As observed from the results, MOC-MAB beats all other algorithms in both objectives except CD-UCB1. While the regret of CD-UCB1 in the dominant objective is slightly better than that of MOC-MAB, its regret is much worse than MOC-MAB in the non-dominant objective. This is expected since it only aims to maximize the reward in the dominant objective without considering the other objective. 

\subsection{Experiment 2 - Multichannel Communication} \label{sec:simchannel}

\rr{In this experiment, we consider the multichannel communication application given in Section \ref{sec:applications} with ${\cal Q} = \{1,2\}$, ${\cal R} = \{1,0.5,0.25,0.1\}$ and $T=10^6$. The channel gain for channel $Q$ in round $t$, denoted by $h^2_{Q,t}$ is independently sampled from the exponential distribution with parameter $\lambda_{Q}$, where 
$[\lambda_1,\lambda_2] = [0.25 ~ 0.25]$. The type of the distributions and the parameters are unknown to the user. $\text{SNR}_{Q,t}$ is sampled from the uniform distribution over $[0,5]$ independently for both channels. In this case, the outage event for transmission rate-channel pair $(R,Q)$ in round $t$ is defined as $\log_2 (1 + h^2_{Q,t} \text{SNR}_{Q,t} ) < R$.

Every algorithm is run $20$ times and the results are averaged over these runs. Simulation results given in Fig. \ref{fig:ads} show the total reward of the algorithms in both objectives as a function of rounds. 
As observed from the results, there is no algorithm that beats MOC-MAB in both objectives. 
In the dominant objective, the total reward of MOC-MAB is 8.21\% higher than that of CP-UCB1, 10.59\% higher than that of CS-UCB1, 21.33\% higher than that of P-UCB1 and 82.94\% higher than that of S-UCB1 but 8.52\% lower than that of CD-UCB1. Similar to Experiment 1, we expect the total reward of CD-UCB1 to be higher than MOC-MAB because it neglects the non-dominant objective.
On the other hand, in the non-dominant objective, MOC-MAB achieves total reward $13.66\%$ higher than that of CD-UCB1.}

\begin{figure}[h!] 
	
	\begin{minipage}[b]{1.0\linewidth}
		\centering
		\centerline{\includegraphics[width=8.5cm]{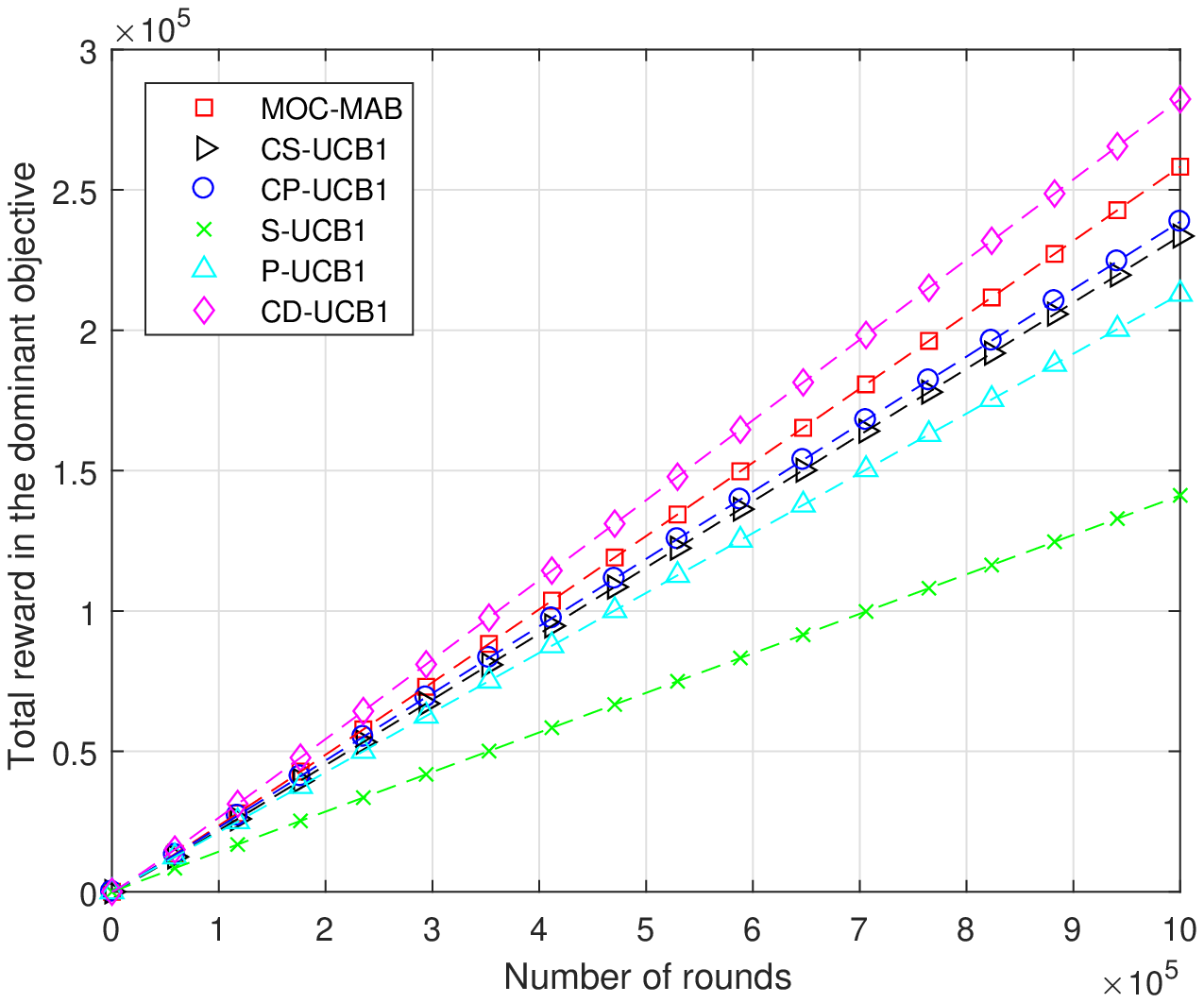}}
		
	\end{minipage}
	
	\begin{minipage}[b]{1.0\linewidth}
		\centering
		\centerline{\includegraphics[width=8.5cm]{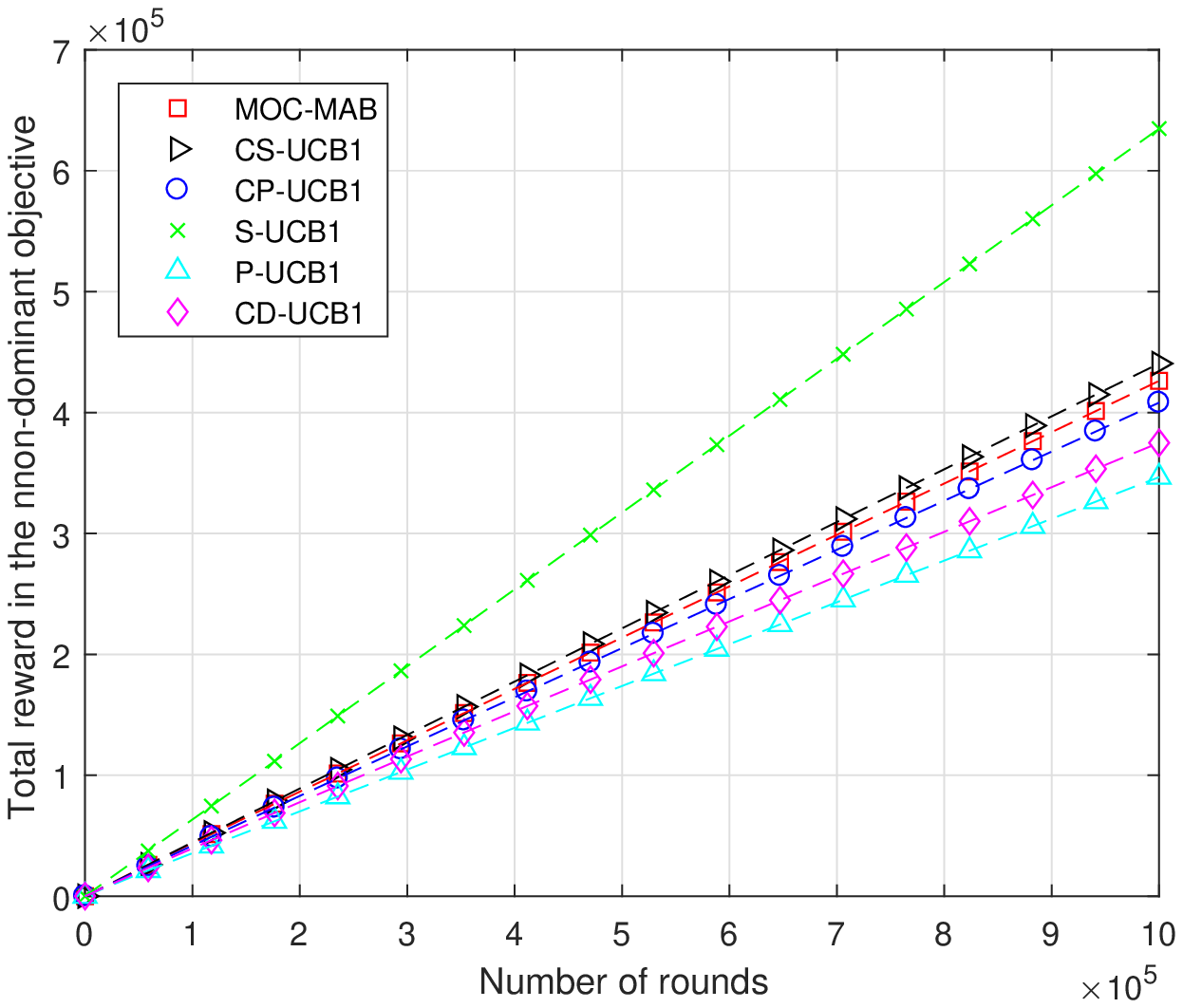}}
		
	\end{minipage}
	
	\caption{Total rewards of MOC-MAB and the other algorithms for Experiment 2.}
	\label{fig:ads}
	
\end{figure}

\subsection{Experiment 3 - Display Advertising} \label{sec:recsys}

\rr{
In this experiment, we consider a simplified display advertising model where in each round $t$ a user with context $x^{\text{usr}}_t$ visits a publisher's website, an ad with context $x^{\text{ad}}_t$ arrives to an advertiser, which together constitute the context $x_t = (x^{\text{usr}}_t, x^{\text{ad}}_t)$. Then, the advertiser decides whether to display the ad on the publisher's website (indicated by action $a$) or not (indicated by action $b$). 
The advertiser makes a unit payment to the publisher for each displayed ad (pay-per-view model). The first objective is related to the click through rate and the second objective is related to the average payment. Essentially, when action $a$ is taken in round $t$, then $r^2_t = 0$, and $r^1_t = 0$ if the user does not click to the ad and $r^1_t=1$ otherwise. When action $b$ is taken in round $t$, the reward is always $(r^1_t,r^2_t) = (0,1)$.

We simulate the model described above by using the Yahoo! Webscope dataset R6A,\footnote{http://webscope.sandbox.yahoo.com/} which consists of over 45 million visits to the Yahoo! Today module during 10 days. This dataset was collected from a personalized news recommender system where articles were displayed to users with a picture, title and a short summary, and the click events were recorded. In essence, the dataset only contains a set of continuous features derived from users and news articles by using conjoint analysis and the click events \cite{chu2009case}. Thus, for our illustrative result, we adopt the feature of the news article as the feature of the ad and the click event as the event that the user clicks to the displayed ad.  

We consider the data collected in the first day which consists of around $4.5$ million samples. Each user and item is represented by $6$ features, one of which is always $1$. We discard the constant features and apply PCA to produce two-dimensional user and item contexts. PCA is applied over all user features to obtain the two-dimensional user contexts $x^{\text{usr}}_t$. To obtain the add contexts $x^{\text{ad}}_t$, we first identify the number of ads with unique features, and then, apply PCA over these. 
The total number of clicks on day $1$ is only \rr{$4.07\%$} of the total number of user-ad pairs. Since the click events are scarce, the difference between the empirical rewards of actions $a$ and $b$ in the dominant objective is very small. Thus, we set $\beta=0.1$ in MOC-MAB in order to further decrease uncertainty in the first objective.  

\begin{figure}[h!] 
	
	\begin{minipage}[b]{1.0\linewidth}
		\centering
		\centerline{\includegraphics[width=7.5cm]{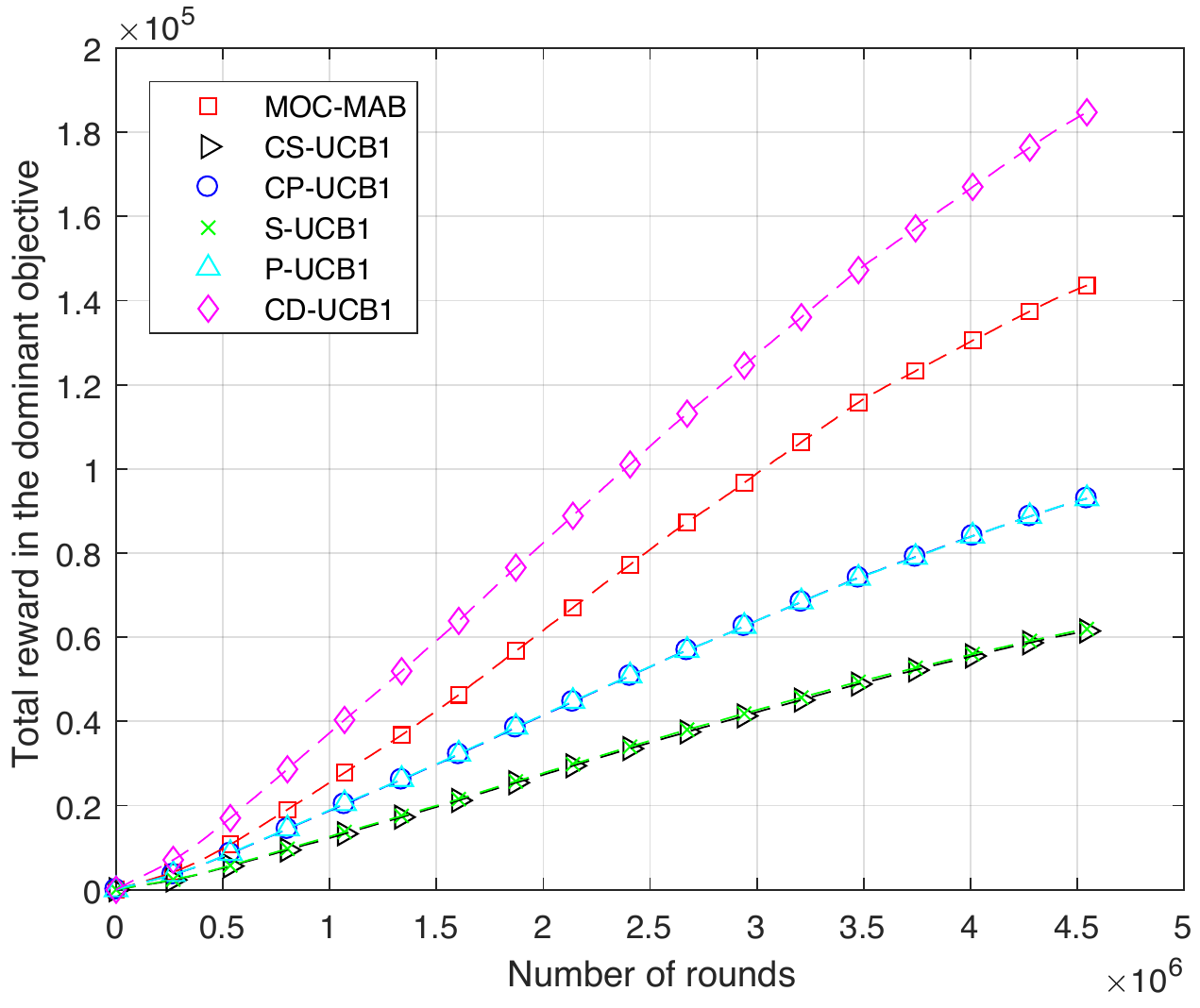}}
		
	\end{minipage}
	
	\begin{minipage}[b]{1.0\linewidth}
		\centering
		\centerline{\includegraphics[width=7.5cm]{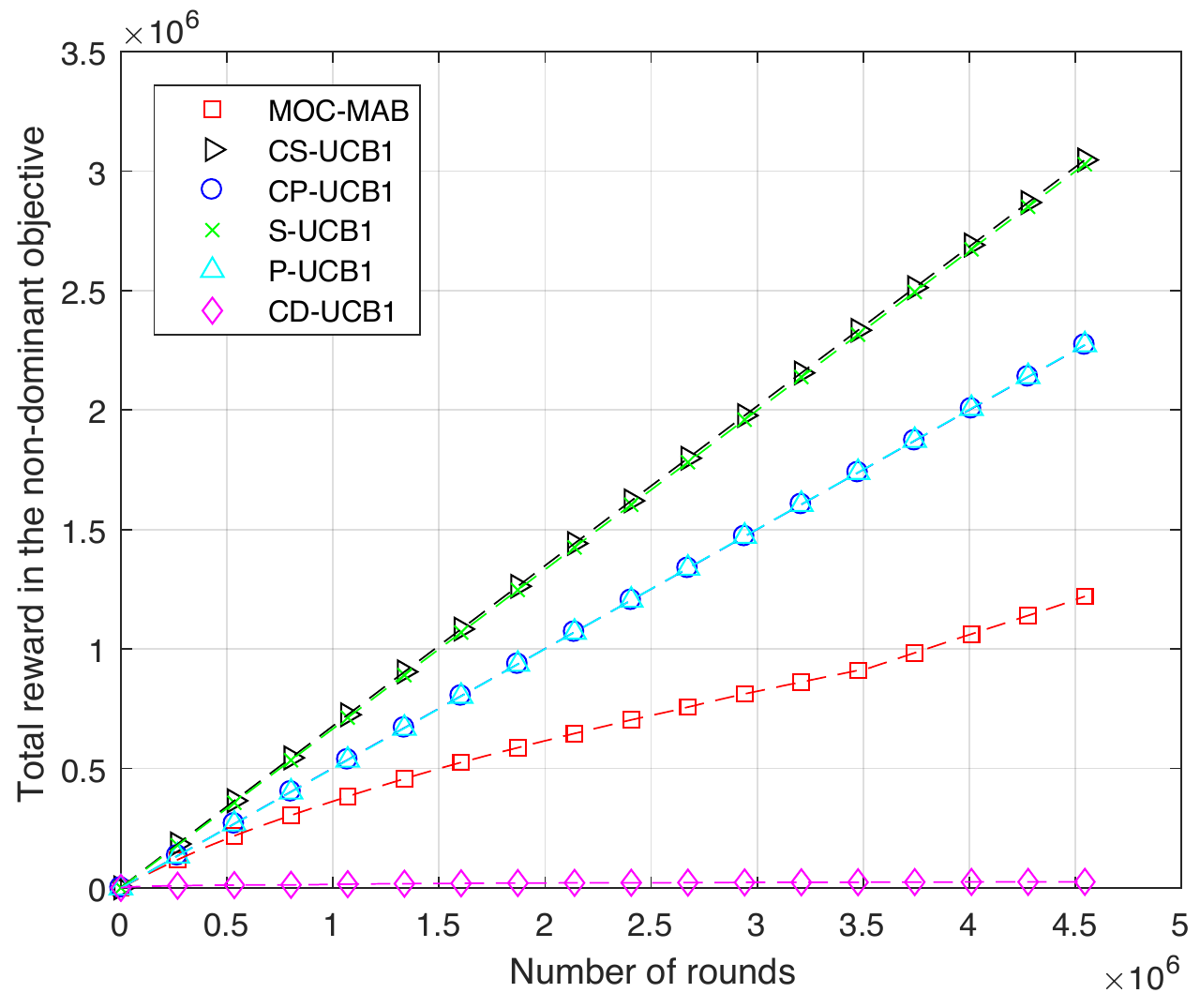}}
		
	\end{minipage}
	
	\caption{Total rewards of MOC-MAB and the other algorithms for Experiment 3.}
	\label{fig:ads2}
	
\end{figure}

Simulation results given in Fig. \ref{fig:ads2} show the total reward of the algorithms in both objectives as a function of rounds. 
In the dominant objective, the total reward of MOC-MAB is 54.5\% higher than that of CP-UCB1, 133.6\% higher than that of CS-UCB1, 54.5\% higher than that of P-UCB1 and 131.8\% higher than that of S-UCB1 but 22.3\% lower than that of CD-UCB1. 
In the non-dominant objective, the total reward of MOC-MAB is 46.3\% lower than that of CP-UCB1, 60\% lower than that of CS-UCB1, 46.3\% lower than that of P-UCB1, 59.7\% lower than that of S-UCB1 and 4751.9\% higher than that of CD-UCB1. As seen from these results, there is no algoritm that outperforms MOC-MAB in both objectives. Although CD-UCB1 outperforms MOC-MAB in the first objective, its total reward in the second objective is much less than the total reward of MOC-MAB.
}

\section{Conclusion}\label{sec:cnc}

In this paper, we propose a new contextual MAB problem with two objectives in which one objective is dominant and the other is non-dominant. According to this definition, we propose two performance metrics: the 2D regret (which is multi-dimensional) and the Pareto regret (which is scalar). Then, we propose an online learning algorithm called MOC-MAB and show that it achieves 
sublinear 2D regret and Pareto regret. To the best of our knowledge, our work is the first to consider a multi-objective contextual MAB problem where the expected arm rewards and contexts are related through similarity information.
We also evaluate the performance of MOC-MAB on both synthetic and real-world datasets and compare it with offline methods and other MAB algorithms. Our results demonstrate that MOC-MAB outperforms its competitors, which are not specifically designed to deal with problems involving dominant and non-dominant objectives.

\begin{appendices}

\section{Proof of Lemma \ref{lemma:prUC}}

From the definitions of $L^i_{a,p}(t)$, $U^i_{a,p}(t)$ and $\text{UC}^i_{a,p}$, it can be observed that the event $\text{UC}^i_{a,p}$ happens when $\mu^i_{a}(\tilde{x}_{p}(t))$ does not fall into the confidence interval $[ L^i_{a,p}(t) -v , U^i_{a,p}(t) +v ]$ for some $t$.
The probability of this event could be easily bounded by using the concentration inequality given in Appendix \ref{app:concentration}, if the expected reward from the same arm did not change over rounds. However, this is not the case in our model since the elements of $\{ \tilde{x}_{p}(t) \}_{t=1}^{N_{p}(T)}$ are not identical which makes the distributions of $\tilde{R}^i_{a,p}(t) $, $t \in \{1, \ldots, N_{p}(T) \}$ different.

In order to resolve this issue, we propose the following: Recall that
\begin{align}
\tilde{R}^i_{a,p}(t) = \rev{\mu^i_{a}( \tilde{x}_{p}(t)) + \tilde{\kappa}^i_p(t)}   \notag
\end{align}
and
\begin{align}
\tilde{\mu}^i_{a,p}(t) 
= 
\rev{ \frac{ \sum_{l=1}^{t-1}  \tilde{R}^i_{a,p}(l) \text{I}( \tilde{a}_p(l) = a)} {\tilde{N}_{a,p}(t)} }       \notag
\end{align}
when $\tilde{N}_{a,p}(t) >0$. Note that when $\tilde{N}_{a,p}(t) = 0$, we have $\tilde{\mu}^i_{a,p}(t) = 0$.
We define two new sequences of random variables, whose sample mean values will lower and upper bound $\tilde{\mu}^i_{a,p}(t)$. The {\em best sequence} is defined as 
$\{  \overline{R}^i_{a,p}(t) \}_{t=1}^{ N_{p}(T) }$ where
\begin{align}
 \overline{R}^i_{a,p}(t) =   \overline{\mu}^i_{a,p} + \tilde{\kappa}^i_p(t)     \notag
\end{align}
and the {\em worst sequence} is defined as 
$\{  \underline{R}^i_{a,p}(t) \}_{t=1}^{ N_{p}(T) }$ where
\begin{align}
 \underline{R}^i_{a,p}(t) = \underline{\mu}^i_{a,p} + \tilde{\kappa}^i_p(t)    .    \notag
\end{align}
Let 
\begin{align}
 \overline{\mu}^i_{a,p}(t) &:= \sum_{l=1}^{t-1} \overline{R}^i_{a,p}(l) \text{I}( \tilde{a}_p(l) = a) / \tilde{N}_{a,p}(t) \notag \\
\underline{\mu}^i_{a,p}(t) &:= \sum_{l=1}^{t-1} \underline{R}^i_{a,p}(l) \text{I}( \tilde{a}_p(l) = a)/ \tilde{N}_{a,p}(t)  \notag 
\end{align}
for $\tilde{N}_{a,p}(t) >0$ and $\overline{\mu}^i_{a,p}(t) = \underline{\mu}^i_{a,p}(t) = 0$ for $\tilde{N}_{a,p}(t) = 0$.
We have
\begin{align}
\underline{\mu}^i_{a,p}(t) \leq \tilde{\mu}^i_{a,p}(t) \leq  \overline{\mu}^i_{a,p}(t)  
~~\forall t \in \{1,\ldots,N_{p}(T)  \}  \notag
\end{align}
almost surely. 

Let
\begin{align}
\overline{L}^i_{a,p}(t) &:=  \overline{\mu}^i_{a,p}(t) - \tilde{u}_{a,p}(t)       \notag \\
\overline{U}^i_{a,p}(t) &:=  \overline{\mu}^i_{a,p}(t) + \tilde{u}_{a,p}(t)     \notag \\
\underline{L}^i_{a,p}(t) &:=   \underline{\mu}^i_{a,p}(t) - \tilde{u}_{a,p}(t)      \notag \\
\underline{U}^i_{a,p}(t) &:=  \underline{\mu}^i_{a,p}(t) + \tilde{u}_{a,p}(t)    .  \notag
\end{align}
Note that $\Pr( \mu^i_{a}( \tilde{x}_{p}(t) ) \notin [L^i_{a,p}(t)- v, U^i_{a,p}(t) + v]  ) = 0$ for $N_{a,p}(t) = 0$ since we have $L^i_{a,p}(t) = -\infty$ and $U^i_{a,p}(t) = + \infty$ when $N_{a,p}(t) = 0$.
Thus, in the rest of the proof, we focus on the case when $N_{a,p}(t) >0$. It can be shown that 
\begin{align}
& \{ \mu^i_{a}( \tilde{x}_{p}(t) ) \notin [L^i_{a,p}(t)- v, U^i_{a,p}(t) + v]  \}  \notag \\
& \subset \{ \mu^i_{a}( \tilde{x}_{p}(t) ) 
\notin [\overline{L}^i_{a,p}(t)  - v, 
            \overline{U}^i_{a,p}(t)  + v]  \}  \notag \\
&\cup  \{ \mu^i_{a}( \tilde{x}_{p}(t) ) 
\notin [\underline{L}^i_{a,p}(t)   - v, 
            \underline{U}^i_{a,p}(t)  + v]  \} . \label{eqn:unionbound1}
\end{align}
The following inequalities can be obtained from the H\"{o}lder continuity assumption:
\begin{align}
& \mu^i_{a}( \tilde{x}_{p}(t) )  \leq \overline{\mu}^i_{a,p}
 \leq \mu^i_{a}( \tilde{x}_{p}(t) ) + L \left( \frac{\sqrt{d}}{m} \right)^{\alpha}  \label{eqn:bestbound} \\
& \mu^i_{a}( \tilde{x}_{p}(t) )  - L \left( \frac{\sqrt{d}}{m} \right)^{\alpha}   \leq \underline{\mu}^i_{a,p}
 \leq \mu^i_{a}( \tilde{x}_{p}(t) ) . \label{eqn:worstbound}  
\end{align}

Since $v = L \left( \sqrt{d}/m \right)^{\alpha}$, using \eqref{eqn:bestbound} and \eqref{eqn:worstbound} it can be shown that
\begin{align}
(i) ~~ \{ \mu^i_{a}( \tilde{x}_{p}(t) ) \notin& [\overline{L}^i_{a,p}(t)  - v, \overline{U}^i_{a,p}(t)    + v]  \}   \notag \\
&\subset  \{ \overline{\mu}^i_{a,p} \notin [\overline{L}^i_{a,p}(t)  , \overline{U}^i_{a,p}(t) ]  \}, \notag \\
(ii) ~~ \{ \mu^i_{a}( \tilde{x}_{p}(t) ) \notin& [\underline{L}^i_{a,p}(t) - v, \underline{U}^i_{a,p}(t)  + v]  \}  \notag \\
&\subset  \{ \underline{\mu}^i_{a,p} \notin [\underline{L}^i_{a,p}(t) , \underline{U}^i_{a,p}(t) ]  \}  .  \notag
\end{align}
Plugging these into \eqref{eqn:unionbound1}, we get
\begin{align*}
&\{ \mu^i_{a}( \tilde{x}_{p}(t) ) \notin [L^i_{a,p}(t)- v, U^i_{a,p}(t) + v]  \}  \notag \\
&\subset \{ \overline{\mu}^i_{a,p} \notin [\overline{L}^i_{a,p}(t)  , \overline{U}^i_{a,p}(t) ]  \}
\cup \{ \underline{\mu}^i_{a,p} \notin [\underline{L}^i_{a,p}(t) , \underline{U}^i_{a,p}(t) ]  \}  .
\end{align*}
Then, using the equation above and the union bound, we obtain
\begin{align}
\Pr( \text{UC}^i_{a,p} ) 
&\leq \Pr \left( \bigcup_{t=1}^{ N_{p}(T) } \{ \overline{\mu}^i_{a,p}  \notin [\overline{L}^i_{a,p}(t)  , \overline{U}^i_{a,p}(t) ]  \}  \right)   \notag \\
&+ \Pr \left( \bigcup_{t=1}^{ N_{p}(T) } \{ \underline{\mu}^i_{a,p} 
\notin [\underline{L}^i_{a,p}(t)  , \underline{U}^i_{a,p}(t)]  \} \right) . \notag
\end{align}
Both terms on the right-hand side of the inequality above can be bounded using the concentration inequality in Appendix \ref{app:concentration}. Using 
$\delta = 1/ (4 |{\cal{A}}| m^{d} T)$ in Appendix \ref{app:concentration} gives
\begin{align*}
\Pr( \text{UC}^i_{a,p} ) \leq \frac{1}{  2 |{\cal A}|  m^{d} T  } 
\end{align*}
since $1 + N_{a,p}(T) \leq T$. 
Then, using the union bound, we obtain
\begin{align}
\Pr( \text{UC}^i_p ) \leq \frac{1}{  2  m^{d} T  }       \notag
\end{align}
and
\begin{align}
\Pr( \text{UC}_p ) \leq \frac{1}{  m^{d} T  }   .    \notag
\end{align}

\section{Proof of Lemma \ref{lemma:instreg1d}}

Let ${\cal T}_{a,p} := \{ 1 \leq l \leq N_{p}(t) : \tilde{a}_{p}(l) = a  \}$ and 
$\tilde{{\cal T}}_{a,p} := \{ l \in {\cal T}_{a,p}: \tilde{N}_{a,p}(l) \geq 1 \}$. By Lemma \ref{lemma:indexdiff}, we have
\begin{align}
\text{Reg}^1_{p}(t) &= \sum_{a \in {\cal A}} \ \sum_{l \in {\cal T}_{a,p} }
 \left( \mu^{1}_{*}( \tilde{x}_{p}(l) ) - \mu^{1}_{ \tilde{a}_{p}(l) }( \tilde{x}_{p}(l) )   \right)
   \notag \\
 & \leq   
  \sum_{a \in {\cal A}} \sum_{l \in \tilde{{\cal T}}_{a,p} } 
  \left(  U^1_{\tilde{a}_{p}(l),p}(l) - L^1_{\tilde{a}_{p}(l),p}(l) + 2 (\beta + 2) v \right) 
   \notag \\
&+ |{\cal A} |C^1_{\max} \notag \\
& \leq    \sum_{a \in {\cal A}} \sum_{l \in \tilde{{\cal T}}_{a,p}  }
  \left(  U^1_{\tilde{a}_{p}(l),p}(l) - L^1_{\tilde{a}_{p}(l),p}(l) \right) 
      \notag \\
&+ 2 (\beta + 2) v N_{p}(t) + |{\cal A} |C^1_{\max} . \label{eqn:composedregret1}
\end{align}
We also have
\begin{align}
\sum_{a \in {\cal A}} \sum_{l \in \tilde{{\cal T}}_{a,p} }&
  \left(  U^1_{\tilde{a}_{p}(l),p}(l) - L^1_{\tilde{a}_{p}(l),p}(l) \right)   \notag \\
  & \leq   \sum_{a \in {\cal A}} \left( B_{m,T} \sum_{ l \in \tilde{{\cal T}}_{a,p}  }  \sqrt{ \frac{1}{\tilde{N}_{a,p}(l) } } \right)   \notag \\
  & \leq B_{m,T} \sum_{a \in {\cal A}} \sum_{k=0}^{ N_{a,p}(t) - 1 } \sqrt{ \frac{1}{1 + k} } \notag \\
  & \leq 2 B_{m,T} \sum_{a \in {\cal A}} \sqrt{ N_{a,p}(t)  } \label{eqn:dec2} \\
  & \leq 2 B_{m,T} \sqrt{   |{\cal A}|  N_{p}(t)   } \label{eqn:cauchyineq2} 
\end{align}
where $B_{m,T} = 2\sqrt{2 A_{m,T}}$, and \eqref{eqn:dec2} follows from the fact that 
\begin{align}
\sum_{ k=0}^{ N_{a,p}(t) - 1  } \sqrt{ \frac{1}{1 + k} } 
\leq \int_{x=0}^{N_{a,p}(t)} \frac{1}{\sqrt{x}} dx = 2 \sqrt{ N_{a,p}(t)  }   .  \notag
\end{align}
Combining \eqref{eqn:composedregret1} and \eqref{eqn:cauchyineq2}, we obtain that on event $\text{UC}^c$
\begin{align}
\text{Reg}^1_{p}(t)
& \leq 
 |{\cal A}|C^1_{\max} + 2 B_{m,T} \sqrt{   |{\cal A}|  N_{p}(t)   } + 2 (\beta + 2) v N_{p}(t) .  \notag
\end{align}

\section{Proof of Lemma \ref{lemma:instreg2n}}

Using the result of Lemma \ref{lemma:multi3}, the contribution to the regret of the non-dominant objective in rounds for which 
$\tilde{u}_{\hat{a}^*_1(t),p}(t) > \beta v$ is bounded by
\begin{align}
C^2_{\max} |{\cal A}| \left( \frac{2 A_{m,T}} {\beta^2 v^2} + 1 \right) . \label{eqn:secondregret1}
\end{align}
Let 
${\cal T}^2_{a,p} := \{ l \leq N_{p}(t) : \tilde{a}_{p}(l) = a
 \text{ and } \tilde{N}_{a,p}(l) \geq 2 A_{m,T} / (\beta^2 v^2)   \}$. 
 By Lemma \ref{lemma:multi2}, we have
\begin{align}
&
 \sum_{a \in {\cal A}} \sum_{l \in {\cal T}^2_{a,p} }
 \left( \mu^{2}_{*}( \tilde{x}_{p}(l) ) - \mu^{2}_{ \tilde{a}_{p}(l) }( \tilde{x}_{p}(l) )   \right)
 \notag \\
 & \leq 
  \sum_{a \in {\cal A}} \sum_{l \in {\cal T}^2_{a,p} }
  \left(  U^2_{\tilde{a}_{p}(l),p}(l) - L^2_{\tilde{a}_{p}(l),p}(l) + 2 v \right) 
\notag \\
& \leq   
  \sum_{a \in {\cal A}} \sum_{l \in {\cal T}^2_{a,p}  }
  \left(  U^2_{\tilde{a}_{p}(l),p}(l) - L^2_{\tilde{a}_{p}(l),p}(l) \right) 
  + 2 v N_{p}(t) .  \label{eqn:composedmetric2}
\end{align}       
We have on event $\text{UC}^c$ 
\begin{align}
\sum_{a \in {\cal A}} \sum_{l \in {\cal T}^2_{a,p} }&
  \left(  U^2_{\tilde{a}_{p}(l),p}(l) - L^2_{\tilde{a}_{p}(l),p}(l) \right)  \notag \\ 
  & \leq   \sum_{a \in {\cal A}} \left( \rev{B_{m,T}}  \sum_{ l \in {\cal T}^2_{a,p}  }  \sqrt{ \frac{1}{\tilde{N}_{a,p}(l) } } \right)   \notag \\
  & \leq B_{m,T} \sum_{a \in {\cal A}} \sum_{k=0}^{ N_{a,p}(t) - 1 } \sqrt{ \frac{1}{1 + k} } \notag \\
  & \leq 2 B_{m,T} \sum_{a \in {\cal A}} \sqrt{ N_{a,p}(t)  } \notag \\
  & \leq 2 B_{m,T} \sqrt{   |{\cal A}|  N_{p}(t)   } \label{eqn:metric2cauchyineq} . 
\end{align}
where $B_{m,T} = 2\sqrt{2 A_{m,T}}$.
Combining \eqref{eqn:secondregret1}, \eqref{eqn:composedmetric2} and \eqref{eqn:metric2cauchyineq}, we obtain
\begin{align}
\text{Reg}^2_{p}(t)  \leq&   C^2_{\max} |{\cal A}| \left( \frac{2 A_{m,T}} {\beta^2 v^2} + 1 \right)  +  2 v N_{p}(t) \notag \\
&+ 2 B_{m,T} \sqrt{   |{\cal A}|  N_{p}(t)   } .  \notag
\end{align}

\section{Concentration Inequality 
\cite{abbasi2011improved,russo2014learning}} \label{app:concentration} 
Consider an arm $a$ for which the rewards of objective $i$ are generated by a process $\{ R^i_{a}(t) \}_{t=1}^T$ with $\mu^i_{a}= \mr{E} [R^i_{a}(t)]$, where the noise \rev{$R^i_{a}(t) - \mu^i_{a}$ is conditionally 1-sub-Gaussian}. Let $N_{a}(T)$ denote the number of times $a$ is selected \rev{by the beginning of round $T$}. 
Let $\hat{\mu}_{a}(T) = \sum_{t=1}^{T-1} \mr{I} (a(t) =a ) R^i_{a}(t) / N_a(T)$ for $N_a(T) >0$ and $\hat{\mu}_{a}(T) = 0$ for $N_a(T) = 0$.
Then, for any $0 < \delta < 1$ with probability at least $1-\delta$ we have


\begin{align}
&\left| \hat{\mu}_{a}(T)  - \mu_a \right| \notag \\
& \leq \sqrt{  \frac{2}{N_a(T)} 
\left(       
1 + 2 \log \left(  \frac{ (1 + N_a(T) )^{1/2} } {\delta}    \right)  
 \right)  }  ~~ \forall T \in \mathbb{N}.   \notag
\end{align}

\end{appendices}

\bibliographystyle{IEEEbib}
\bibliography{OSA}

\end{document}